\documentclass{article} %
\usepackage{iclr2026_conference,times}

\usepackage{amsmath,amsfonts,bm}

\def\eqref#1{equation~\ref{#1}}

\def\1{\bm{1}}

\def\mZ{{\bm{Z}}}

\DeclareMathAlphabet{\mathsfit}{\encodingdefault}{\sfdefault}{m}{sl}
\SetMathAlphabet{\mathsfit}{bold}{\encodingdefault}{\sfdefault}{bx}{n}

\DeclareMathOperator*{\argmax}{arg\,max}

\usepackage{hyperref}
\usepackage{url}
\usepackage{graphicx}
\usepackage{booktabs}
\usepackage{multirow}
\usepackage{subcaption}
\usepackage{cleveref}
\newtheorem{proposition}{Proposition}
\newtheorem{assumption}{Assumption}
\usepackage{amssymb}
\usepackage{enumitem}
\usepackage{wrapfig}

\newcommand{\mypara}[2][\vspace{0.2em}]{#1\noindent\textbf{#2}}

\newcommand{\passof}[1]{\ensuremath{\mathrm{pass@}{#1}}}
\newcommand{\Passof}[1]{\ensuremath{\mathrm{Pass@}{#1}}}
\newcommand{\x}[0]{\ensuremath{\mathbf{x}}}
\newcommand{\y}[0]{\ensuremath{\mathbf{y}}}

\renewcommand{\mZ}[0]{\ensuremath{\mathcal{Z}}}
\newcommand{\yklist}[0]{\ensuremath{\y{}_1,\ldots,\y{}_k}}

\newcommand{\piof}[1]{\ensuremath{\pi_{\mathrm{#1}}^{(k)}}}
\newcommand{\outline}[0]{\ensuremath{\mathbf{o}}}

\newcommand{\oklist}[0]{\ensuremath{\outline{}_1,\ldots,\outline{}_k}}

\newcommand{\pip}[0]{\ensuremath{\pi_p}}
\newcommand{\pie}[0]{\ensuremath{\pi_e}}

\usepackage{soul}
\usepackage[table]{xcolor}
\usepackage{listings}

\lstdefinestyle{promptstyle}{
    basicstyle=\ttfamily\scriptsize,
    breaklines=true,
    breakindent=0pt,
    breakautoindent=false,
    breakatwhitespace=false,
    columns=fullflexible,
    keepspaces=true,
    frame=single,
    rulecolor=\color{black!25},
    backgroundcolor=\color{black!2},
    xleftmargin=6pt,
    xrightmargin=6pt,
    aboveskip=8pt,
    belowskip=8pt,
    showstringspaces=false
}

\title{Planned Test-Time Scaling with Coordinated Reasoning Paths}

\author{Xueqing Wu$^{1}$, Langxing Bai$^{1}$, Hritik Bansal$^{1}$, Po-Nien Kung$^{1}$, Shuo Li$^{2}$, Hao Liu$^{2}$, \\
\textbf{Nanyun Peng}$^{1}$, \textbf{Kai-Wei Chang}$^{1}$ \\
$^{1}$University of California, Los Angeles \quad $^{2}$Amazon \\
\url{https://github.com/shirley-wu/planned-test-time-scaling}
}

\iclrfinalcopy %
\begin{document}

\maketitle
\vspace{-1em}

\begin{abstract}
\vspace{-.5em}
Test-time scaling with parallel branches is widely adopted to improve performance on challenging reasoning tasks. The predominant approach, repeated sampling, draws branches independently from a single policy, which can produce redundant attempts and thereby limit the gains from additional inference compute.
To address this limitation, we propose \textbf{Planned Test-Time Scaling} (PTTS), which replaces independent sampling with a coordinated joint policy: a \textit{planner} generates a solution outline for each branch, steering the branches toward distinct reasoning paths, and an \textit{executor} produces a full solution conditioned on each outline.
Formally, we show that PTTS strictly generalizes repeated sampling and, in a stylized setting, provably promotes coverage of complementary reasoning modes and yields better \passof{k} scaling.
We instantiate PTTS on top of strong reasoning models, keeping them fixed as executors while replacing repeated sampling with  PTTS inference to further enhance test-time scaling.
Concretely, we develop two variants: \textbf{PTTS-ZS} prompts a model to jointly generate outlines for all branches in a single autoregressive pass, while \textbf{PTTS-RL} directly optimizes the planner against the \passof{k} reward using truncated execution rollouts for efficient training and a sharper reward signal.
Across five mathematical reasoning benchmarks with Qwen3-1.7B and 4B, PTTS-ZS improves \passof{64} over repeated sampling by up to 6.7 points, while PTTS-RL further increases the gain to up to 13.4 points. Further analysis indicates that broader coverage of distinct reasoning paths contributes to these gains.
Overall, PTTS provides a general framework for improving test-time scaling by coordinating reasoning branches, with zero-shot and trainable instantiations that yield substantial performance gains.

\end{abstract}
\vspace{-1em}

\begin{figure}[!h]
    \centering
    \includegraphics[width=\linewidth]{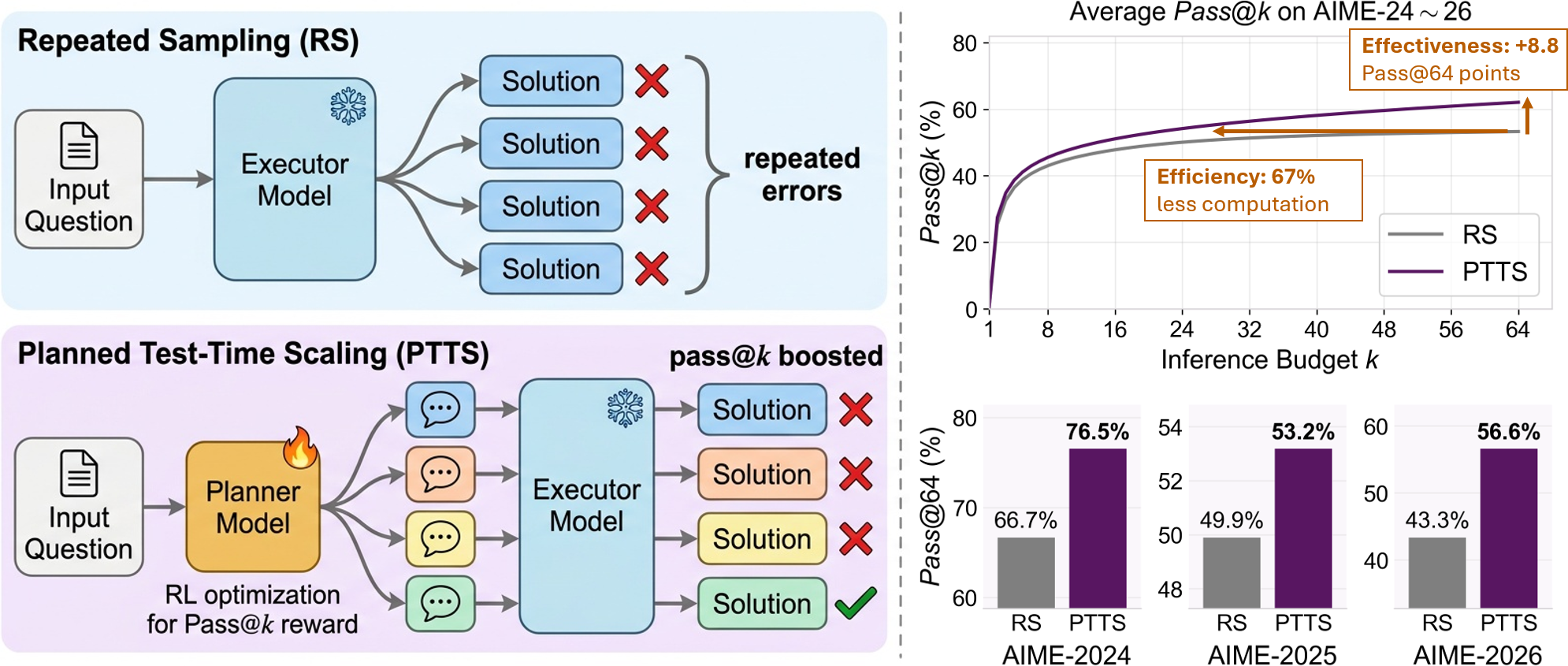}
    \vspace{-1.5em}
\caption{\small \textbf{Left: overview of planned test-time scaling (PTTS).} Unlike repeated sampling (RS) that draws $k$ independent solutions from a frozen executor and often repeats the same errors, PTTS uses a planner to generate $k$ diverse outlines and guide the executor, where the planner can be further trained via reinforcement learning to directly optimize \passof{k}. \\[4pt]
\textbf{Right: PTTS performance} with Qwen3-1.7B on AIME-2024--2026. PTTS matches the \passof{64} of RS with 67\% less computation (efficiency) and exceeds it by +8.8 points at $k\!=\!64$ (effectiveness).}
    \label{fig:teaser}
\end{figure}

\section{Introduction}

As large language models (LLMs) are applied to increasingly challenging reasoning tasks, even strong models often fail to solve a problem reliably in a single attempt. Test-time scaling addresses this limitation by allocating additional inference compute \citep{wu2024inference,snell2024scaling,brown2024large}, most commonly through \textit{repeated sampling} \citep{brown2024large}: given a budget of $k$ attempts, the model independently samples $k$ candidate solutions from a single policy, and performance is measured by \passof{k}, i.e., whether at least one attempt is correct.
The benefit of this additional compute therefore depends critically on \textbf{coverage}: additional attempts are useful only insofar as they explore complementary reasoning paths rather than repeat existing ones.
Under repeated sampling, however, independent samples from the same policy tend to concentrate around high-probability reasoning modes, producing redundant attempts.
Reinforcement learning (RL) can further amplify this effect, improving \passof{1} while narrowing the policy distribution and degrading \passof{k} \citep{yuedoes,cui2025entropy,wu2026invisibleleashrlvrescape}.
Unlocking the full value of parallel compute therefore requires explicit coordination that actively steers branches toward distinct reasoning paths rather than relying on stochasticity to provide coverage.

Recent work explicitly coordinates parallel attempts by first constructing diverse high-level reasoning plans through handcrafted pipelines---for example, iterative concept elicitation from an LLM \citep{handa2025guidedsampling} or over-generation followed by clustering of candidate plans \citep{yang2025explore}---and then using these plans to steer different reasoning branches. Their empirical gains suggest that explicit coordination via a planning stage is a promising direction. We formalize this shared structure by treating planning as a  joint policy across branches, enabling both theoretical analysis and direct optimization.

We introduce \textbf{Planned Test-Time Scaling} (PTTS), which factorizes the joint policy over $k$ reasoning branches into a planner-executor decomposition. As shown in Figure \ref{fig:teaser}, a \textit{planner} jointly produces $k$ solution outlines, one for each branch, and an \textit{executor} then generates a full solution for each branch conditioned on its  outline.
By construction, PTTS strictly generalizes repeated sampling and thus guarantees no worse achievable \passof{k}.
In a stylized setting of reasoning-mode selection, we further show that a marginal policy optimized for \passof{1} collapses onto a single dominant mode, causing repeated sampling to replay similar failures and bottlenecking test-time scaling on problems that benefit from complementary reasoning modes.
In contrast, PTTS can be optimized directly for \passof{k}, incentivizing the planner to coordinate branches across complementary reasoning modes and yielding strictly better \passof{k} scaling.

We instantiate PTTS in two complementary forms, PTTS-ZS and PTTS-RL. The zero-shot variant, \textbf{PTTS-ZS}, uses a strong reasoning model as the executor and a \textit{base} LLM as the planner, leveraging the base model's broader generation diversity. Unlike prior approaches that rely on multi-stage pipelines to construct diverse plans, PTTS-ZS generates all $k$ outlines jointly in a single autoregressive pass, yielding a simple yet effective design that naturally enables end-to-end optimization. Building on this formulation, \textbf{PTTS-RL} freezes the executor and directly optimizes the planner with the \passof{k} of the $k$ executor solutions as the reward. To reduce rollout cost and sharpen credit assignment to the planner, we truncate the executor's reasoning budget during training, preventing the executor from rescuing poor outlines through extended rethinking.

Evaluation on five mathematical reasoning benchmarks using Qwen3 models \citep{yang2025qwen3technicalreport} demonstrates substantial gains in both effectiveness and efficiency. PTTS-ZS consistently improves over repeated sampling, yielding up to a 6.7-point gain in \passof{64}, while optimizing the planner with PTTS-RL further increases this gain to as much as 13.4 points. Notably, PTTS-RL surpasses repeated sampling using only half the sampling budget, making it more than 2$\times$ as compute-efficient.
Further analysis points to improved coverage of the reasoning space: PTTS-ZS substantially increases overall reasoning diversity, while PTTS-RL further enhances diversity among branches that reach correct answers.
Finally, PTTS-RL planners can serve as flexible add-ons that transfer to executors unseen during training, including larger reasoning models (+4.4 \passof{64}) and executors trained with diversity-oriented RL (+4.4 \passof{64}), demonstrating that PTTS is complementary to executor-side improvements.

To summarize, our contributions are as follows: (1) We introduce Planned Test-Time Scaling (PTTS), a planner-executor framework that replaces independent repeated sampling with a coordinated joint policy over multiple reasoning branches. (2) We analyze PTTS theoretically, showing that it strictly generalizes repeated sampling and that directly optimizing for \passof{k} promotes coverage of complementary reasoning modes under fixed inference budgets. (3) We instantiate PTTS in zero-shot and RL-based variants, achieving gains of up to 13.4 points over repeated sampling across five mathematical reasoning benchmarks.

\section{Related Work}
\label{sec:related_work}

\mypara[]{Mode collapse in test-time scaling.}
Recent work increasingly uses test-time scaling to enhance off-the-shelf models by allocating additional inference compute \citep{snell2024scaling,wu2024inference,zhang2025survey,brown2024large}. A particularly simple and effective approach is repeated sampling, which independently draws multiple solutions from the same policy \citep{brown2024large}.
However, repeated sampling relies on independently sampled branches from the same policy, so reduced diversity makes additional samples redundant, limiting coverage and performance. This issue is especially pronounced after reinforcement learning with verifiable rewards (RLVR) \citep{shao2024deepseekmath,guo2025deepseek}, which can concentrate the model on a narrow distribution \citep{yuedoes,cui2025entropy,wu2026invisibleleashrlvrescape}, thereby harming \passof{k}.

\mypara{Diversity-aware test-time scaling.}
One line of work addresses this at inference time by coordinating the $k$ branches with a planning stage prior to problem solving. As our primary baseline, Guided Sampling \citep{handa2025guidedsampling} iteratively generates distinct solution concepts to guide the final generation; others sample and cluster high-level plans \citep{yang2025explore} or perturb the query before parallel attempts \citep{wang2025effectsamplingdiversityscaling}. While effective, these methods lack a unified formal framework that can be optimized end-to-end, which we aim to provide. Closely related to our approach, \citet{kang2025road} and \citet{li2026cast}  train planners to propose multiple strategies before solving, but mostly focus on synthetic or coding tasks and lack a joint formulation or theoretical perspective.

\mypara{Diversity-aware RL training.}
Another line of work redesigns RL training to promote exploration behavior and diversity in the marginal policy itself. Prior work directly optimizes estimates of \passof{k} for each branch \citep{tang2025optimizing,chen2025pass,walder2025pass}, maintains entropy to prevent mode collapse \citep{cui2025entropy,wang2026beyond}, promotes semantic diversity \citep{li2025jointly,yao2026diversity}, adopts risk-aware training objectives \citep{ren2025riskpo,jiang2025risk}, or designs structured curricula to encourage exploration \citep{setlur2026e3}. These methods focus on improving the marginal policy and are therefore complementary to our proposed test-time scaling algorithm. As shown in \S\ref{sec:exp:ablations}, our proposed PTTS can be combined with this line of work to further improve performance.

\section{Planned Test-Time Scaling}
\label{sec:method}

This section presents our planned test-time scaling (PTTS) framework and instantiations. We begin with the  problem formulation (\S\ref{sec:method:ptts}), then provide two theoretical perspectives on why PTTS improves \passof{k} (\S\ref{sec:method:advantage}), and finally describe the concrete algorithm (\S\ref{sec:method:method}).

\subsection{From Repeated Sampling to PTTS}
\label{sec:method:ptts}

Consider a verifiable reasoning problem $\x{}$, and let $r(\y{};\x{}) \in \{0,1\}$ indicate whether a candidate solution $\y{}$ is correct. Given a test-time budget of $k$ attempts, we measure success using \passof{k}:
{\small\begin{align*}
    \passof{k}\left(\yklist{};\x{}\right)
    \triangleq
    \max_{i\in\{1,\ldots,k\}} r(\y_i;\x{}).
\end{align*}}

\begin{figure*}[!tbph]
     \centering
     \begin{subfigure}[b]{0.55\textwidth}
         \centering
\includegraphics[width=\textwidth]{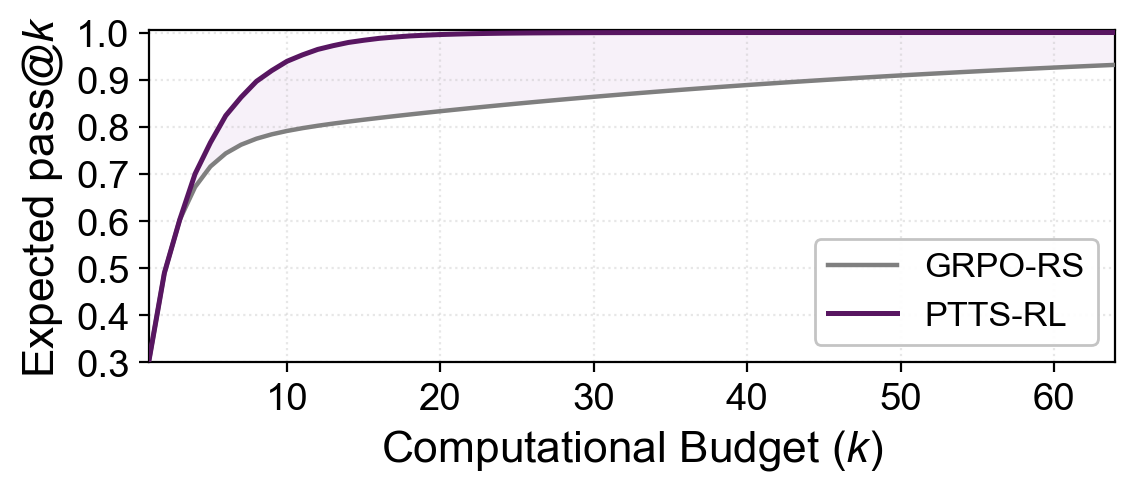}
     \vspace{-1.5em}
         \caption{\small Comparing the \passof{k}  of PTTS-RL against repeated sampling based on GRPO (GRPO-RS).}
         \label{fig:mode_collapse_visualization:1}
     \end{subfigure}
     \hfill
     \begin{subfigure}[b]{0.42\textwidth}
         \centering
\includegraphics[width=\textwidth]{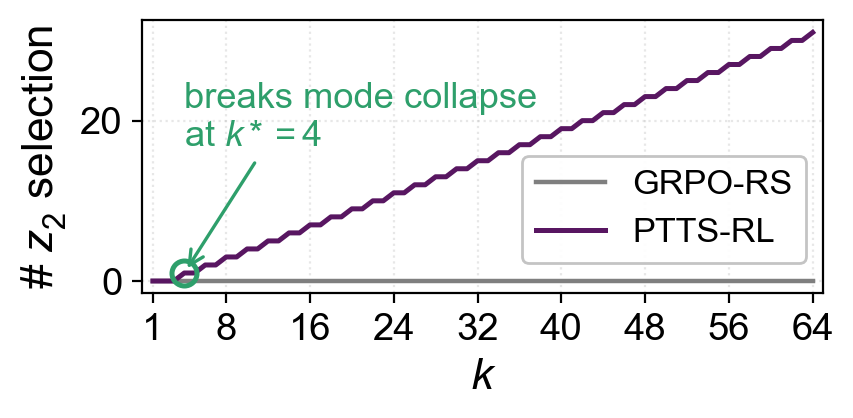}
     \vspace{-1.7em}
\caption{\small Inference budgets (out of $k$) allocated to the non-dominant mode $z_2$, with larger value representing higher diversity.}
         \label{fig:mode_collapse_visualization:2}
     \end{subfigure}
     \vspace{-.4em}
     \caption{\small\textbf{\Passof{k} performance and policy diversity in a synthetic setup} with two data categories and two reasoning modes ($z_1$ being dominant and $z_2$ being non-dominant), as detailed in Appendix \ref{app:math:eg}.}
\label{fig:mode_collapse_visualization}
\vspace{-.8em}
\end{figure*}

\mypara[]{Repeated sampling (RS).}
As the predominant test-time scaling approach, RS draws the $k$ attempts independently from a single base policy $\pi$, inducing the joint distribution:
{\small\begin{align*}
    (\yklist{}) \sim \piof{rs}(\cdot \mid \x{}),
    \qquad
    \piof{rs}(\yklist{} \mid \x{})
    \triangleq
    \prod_{i=1}^k \pi(\y_i \mid \x{}).
\end{align*}}
RS is simple and broadly applicable, but it does not coordinate branches or allocate the budget across complementary reasoning paths. As a result, additional samples may cluster around similar high-probability solutions and repeat the same mistakes.

\mypara{Planned test-time scaling (PTTS).}
PTTS instead treats the $k$ attempts as a structured test-time computation rather than independent samples. Ideally, one would model a joint policy directly over $(\yklist{})\sim \piof{ptts}(\cdot \mid \x{})$.
However, directly modeling an unrestricted joint distribution over $(\yklist{})$ is difficult: the space of solution tuples grows rapidly with $k$, and the policy must capture both solution quality and how attempts should differ.

PTTS therefore uses a \textbf{planner-executor} factorization. The \textit{planner} first produces a shared plan consisting of $k$ solution outlines, $\outline{} = (\oklist{})\sim \pip(\cdot \mid \x{})$, where each outline steers one attempt toward a distinct reasoning path. The \textit{executor} then generates a full solution conditioned on the problem and the corresponding outline: $\y_i \sim \pie(\cdot \mid \x{}, \outline{}_i)$. Equivalently, PTTS defines the following joint distribution over the generated plan and solutions:
{\small\begin{align*}
    \piof{ptts}(\outline{}, \yklist{} \mid \x{})
    \triangleq
    \pip(\outline{} \mid \x{})
    \prod_{i=1}^k
    \pie(\y_i \mid \x{}, \outline{}_i).
\end{align*}}
This factorization makes PTTS tractable while preserving branch-level coordination: the planner chooses complementary high-level reasoning paths, and the executor realizes each path independently. In \S\ref{sec:method:advantage}, we analyze why this structure can improve \passof{k} over RS.

\subsection{Why PTTS Improves Pass@$k$}
\label{sec:method:advantage}

\mypara[]{Expressiveness.}
From a policy-class perspective, $\piof{ptts}$ strictly generalizes $\piof{rs}$: when the planner emits an empty or constant outline for every branch, PTTS reduces to RS. Therefore, the optimal \passof{k} attainable within PTTS is no lower than that attainable within RS.

This gap can be made \textit{strict} under a per-attempt success constraint. Suppose each branch has fixed marginal success probability $p\triangleq\mathbb{E}\left[r(\y_i;\x{})\right]$. For RS, independence gives $\passof{k}=1-(1-p)^k$. In contrast, PTTS can coordinate branches so that their success events are nearly disjoint, for example by assigning complementary strategies. In the ideal disjoint case, $\passof{k}=\min(kp,1)$, which is \textit{strictly larger} than RS for $p\in(0,1)$ and $k>1$. Thus, PTTS converts the same average per-attempt competence into a higher joint pass rate by reducing redundancy across attempts.

\definecolor{outlinegold}{HTML}{BC8310}

\begin{figure}[!tbph]
    \centering
    \includegraphics[width=\linewidth]{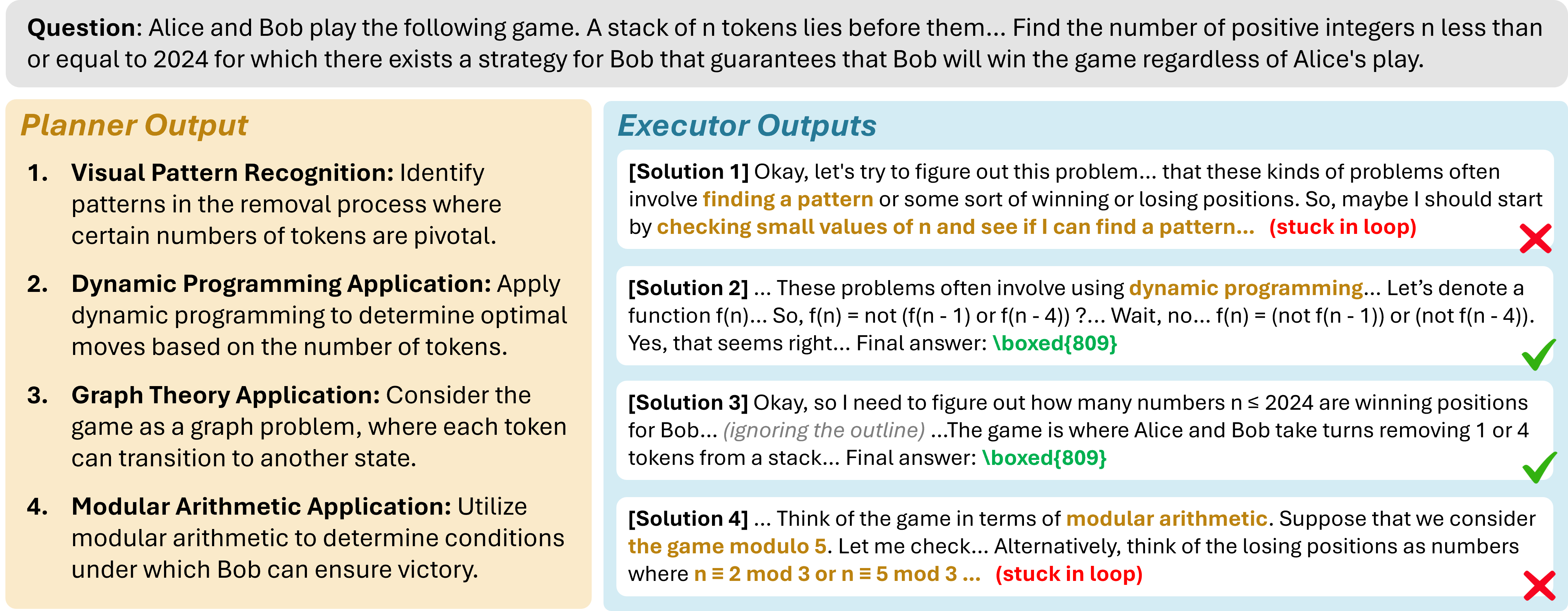}
    \vspace{-1.5em}
    \caption{\small \textbf{Example output from PTTS-ZS.} \textbf{Left:} a markdown list of four distinct outlines generated by the planner $\pi_p$ in a single autoregressive pass. \textbf{Right:} solutions generated by the executor $\pi_e$ conditioned on each outline, with the outline-following portions \textbf{\textcolor{outlinegold}{highlighted}}. Most solutions faithfully follow their outlines, steering the reasoning toward different modes and boosting the diversity.}
    \label{fig:output_example}
\vspace{-1em}
\end{figure}

\mypara{Mode coverage analysis.}
We then take a dataset-level optimization perspective to illustrate why standard paradigms such as GRPO suffer from \textit{mode collapse}, and how PTTS-RL naturally mitigates this issue. Optimizing dataset-level \passof{1} encourages the policy to concentrate on a \textit{dominant reasoning mode}. While this maximizes average single-shot reward, it inherently sacrifices diversity and can ultimately degrade \passof{k} \citep{yuedoes}.

To formalize this intuition, we study a simplified setting in which the executor is fixed and imperfect, reducing the reasoning process to a mode-selection problem. Assume each example belongs to a latent category $c \in \mathcal{C}$, and the policy selects a reasoning mode $z \in \mathcal{Z}$, with success rate $R_{zc}$ determined only by the selected mode $z$ and data category $c$. We consider three realistic assumptions, with the full derivation deferred to Appendix \ref{app:math}:
\begin{enumerate}[nosep,leftmargin=*]
\item \textit{Imperfect policy:} The policy cannot observe the true category $c$ perfectly. Instead, it acts based on a noisy observed category $c'$, making perfect deterministic mode-matching impossible.
\item \textit{Mode dominance:} For a certain observed category $c'$, there exists a single dominant mode $z^*(c')$ that achieves a strictly higher expected success rate than any other mode.
\item \textit{Complementary modes:} The dominant mode $z^*$ struggles on a ``worst-case'' category $c^\dagger$ by a strict margin. Crucially, this worst-case category benefits more from a non-dominant mode.
\end{enumerate}

Under these conditions, optimizing the marginal policy for \passof{1} via GRPO drives fundamentally different behavior than optimizing the joint policy against \passof{k} via PTTS-RL. Since GRPO maximizes expected average return, it assigns all probability mass to the dominant mode $z^*(c')$, leading directly to mode collapse. In contrast, PTTS-RL optimizes directly for \passof{k}, naturally accounting for the \textit{diminishing marginal returns} of repeatedly sampling $z^*$. We show that beyond a finite threshold $k^*$, \textbf{PTTS-RL strategically allocates reasoning slots to non-dominant modes}, thereby rescuing performance on the worst-case category $c^\dagger$. Consequently, these policies exhibit distinct scaling dynamics: PTTS-RL drives the residual error $\epsilon = 1-\passof{k}$ to zero exponentially as $k\to\infty$, whereas GRPO's 
\passof{k} plateaus whenever the dominant mode has zero accuracy on the worst-case category. Even when GRPO does not plateau, \textbf{PTTS-RL shrinks the residual error at a strictly faster exponential rate} by actively covering blind spots rather than repeatedly resampling the dominant mode.

To illustrate these theoretical mechanics, Figure \ref{fig:mode_collapse_visualization} instantiates this model in a minimal two-category, two-mode environment where $z_1$ is dominant. The numerical results confirm our intuition: PTTS-RL actively breaks mode collapse, substantially improving coverage of the non-dominant class (Figure \ref{fig:mode_collapse_visualization:2}) and accelerating \passof{k} improvements after the $k^*$ threshold is crossed and before test-time scaling saturates (Figure \ref{fig:mode_collapse_visualization:1}).

\subsection{Instantiation}
\label{sec:method:method}

\mypara[]{PTTS-ZS.} In the zero-shot setting, we instantiate both the planner $\pi_p$ and the executor $\pi_e$ with LLMs.
As shown in Figure \ref{fig:output_example}, given a problem $\x{}$, the planner generates a markdown list of $k$ solution outlines in a single autoregressive pass, which is then deterministically parsed into $k$ outlines $(\oklist{})$. The executor then independently generates a full solution $\y_i \sim \pie(\cdot \mid \x{}, \outline_i)$ for each outline.
Concretely, we use strong reasoning models such as Qwen3 \citep{yang2025qwen3technicalreport} as executors, making PTTS-ZS an inference-only add-on to further enhance the test-time scaling. However, we find that reasoning models are poorly suited as planners: they tend to commit early to a single strategy and start solving the problem, rather than generating multiple diverse, high-level solution outlines. Therefore, we use the corresponding base models as planners.
Despite its simplicity, PTTS-ZS yields surprisingly large improvements over \passof{k}.

\mypara{PTTS-RL.} We then present PTTS-RL, which trains the planner end-to-end with GRPO \citep{shao2024deepseekmath,yu2026dapo} to directly optimize \passof{k}.

In GRPO, for each input \x{}, we sample $G$ outputs $\{\y{}^{(i)}\}_{i=1}^G$ from the model and obtain a reward $R_i$ for each output. We then compute the advantage $A_i$ by normalizing the rewards within the group: $A_i = \Big( R_i - \text{mean}(\{R_j\}_{j=1}^G) \Big) ~/~ \text{std}(\{R_j\}_{j=1}^G)$. This same advantage is assigned to every token in the corresponding output $\y{}^{(i)}$, i.e., $A_{i,t}=A_i$. The resulting training objective is:
{\small\begin{align*}
\mathcal{J}_\text{GRPO} = \mathbb{E} \left[ \frac{1}{G} \sum_{i=1}^G \frac{1}{|\y{}^{(i)}|} \sum_{t=1}^{|\y{}^{(i)}|} \min\Big( \rho_{i,t} A_{i,t} ,~ \text{clip}(\rho_{i,t} , 1-\epsilon_L, 1+\epsilon_H) A_{i,t} \Big) \right],
\end{align*}}
where $\rho_{i,t} = \pi_\theta\left(\y{}^{(i)}_t|~\x{},{\y{}}^{(i)}_{<t}\right) ~/~ \pi_{\theta_{old}}\left(\y{}^{(i)}_t|~\x{},{\y{}}^{(i)}_{<t}\right)$ is the importance sampling term.

PTTS-RL applies GRPO to the planner model $\pi_p$, where each sampled output is a list of outlines $\outline = (\oklist{})$. The reward for a sampled output is computed using a fixed executor $\pi_e$, which generates a solution $\y_i$ for each of the $k$ outlines, $i=1,\ldots,k$. Each generated solution is evaluated for accuracy as $r(\y_i \mid \x{})\in\{0,1\}$, yielding the \passof{k} reward:
{\small\begin{align*}
     R\left(\outline{}\right) = \max_{i=1,\ldots,k} r(\y_i\mid\x{}),\quad\y_i \sim \pi_e(\cdot\mid\x{},\outline_i).
\end{align*}}
However, generating full-length solutions for all $k$ branches makes training computationally expensive. Moreover, long execution traces allow the executor to drift away from the outline and achieve the correct answer through extensive rethinking, assigning positive reward to a poor outline and weakening the training signal. We therefore adopt a \textit{truncated execution} strategy during training, limiting executor outputs to fewer tokens than at inference ($4k$ vs. $10k$ tokens, as detailed in \S\ref{sec:exp:settings}) to improve efficiency and sharpen the reward signal.

\section{Experiments}

We empirically study three questions: (1) How much does the PTTS framework improve \passof{k}? (2) Which output properties drive these gains, and how do PTTS-ZS and PTTS-RL shape them? (3) Which PTTS design choices and hyperparameters most affect PTTS \passof{k} performance?
In this section, we describe the experimental setup in \S\ref{sec:exp:settings}, address (1)--(2) through the  results and analyses in \S\ref{sec:exp:main}, and study (3) through ablations in \S\ref{sec:exp:ablations}.

\subsection{Setup}
\label{sec:exp:settings}

\mypara[]{Evaluation settings.}
We evaluate on five benchmarks: MATH-500 \citep{hendrycks2measuring}, AIME 2024, 2025, and 2026 \citep{aops_aime_problems_solutions}, and HMMT-Feb26 \citep{dekoninck2026matharena}.
We report \passof{k} for $k$ up to 64. For each problem, we generate $n=64$ solutions, with a budget of $10k$ tokens per solution. We then analytically compute \passof{k} as the expected maximum reward among $k$ responses sampled from the set of $n$ generated responses:
{\small\begin{align*}
    \passof{k} = \mathbb{E} \left[1 -  \binom{n-c}{k} ~/~ \binom{n}{k} \right], \quad c = \sum_{i=1}^n\mathbf{1}\left(r(\y{}_i\mid\x{}) = 1\right).
\end{align*}}

\mypara[]{Models and baselines.} We compare PTTS-ZS and PTTS-RL against two baselines: (1)  repeated sampling, and (2) Guided Sampling \citep{handa2025guidedsampling}, an inference-time method similar to PTTS-ZS that first generates diverse \textit{concepts} and then uses them to guide reasoning.
We conduct experiments with Qwen3 models at 1.7B and 4B scales. For all methods, we use the reasoning model \texttt{Qwen3-1.7B/4B} as the executor. For zero-shot methods that involve planning (PTTS-ZS and Guided Sampling), we use the corresponding base model, \texttt{Qwen3-1.7B/4B-Base}, as the planner. For Guided Sampling, we follow the recommended setting of generating 5 ideas. For PTTS-RL, we train the planner starting from the base model, as discussed below.

\begin{figure}[!tbh]
    \centering
        \centering
        \includegraphics[width=\linewidth]{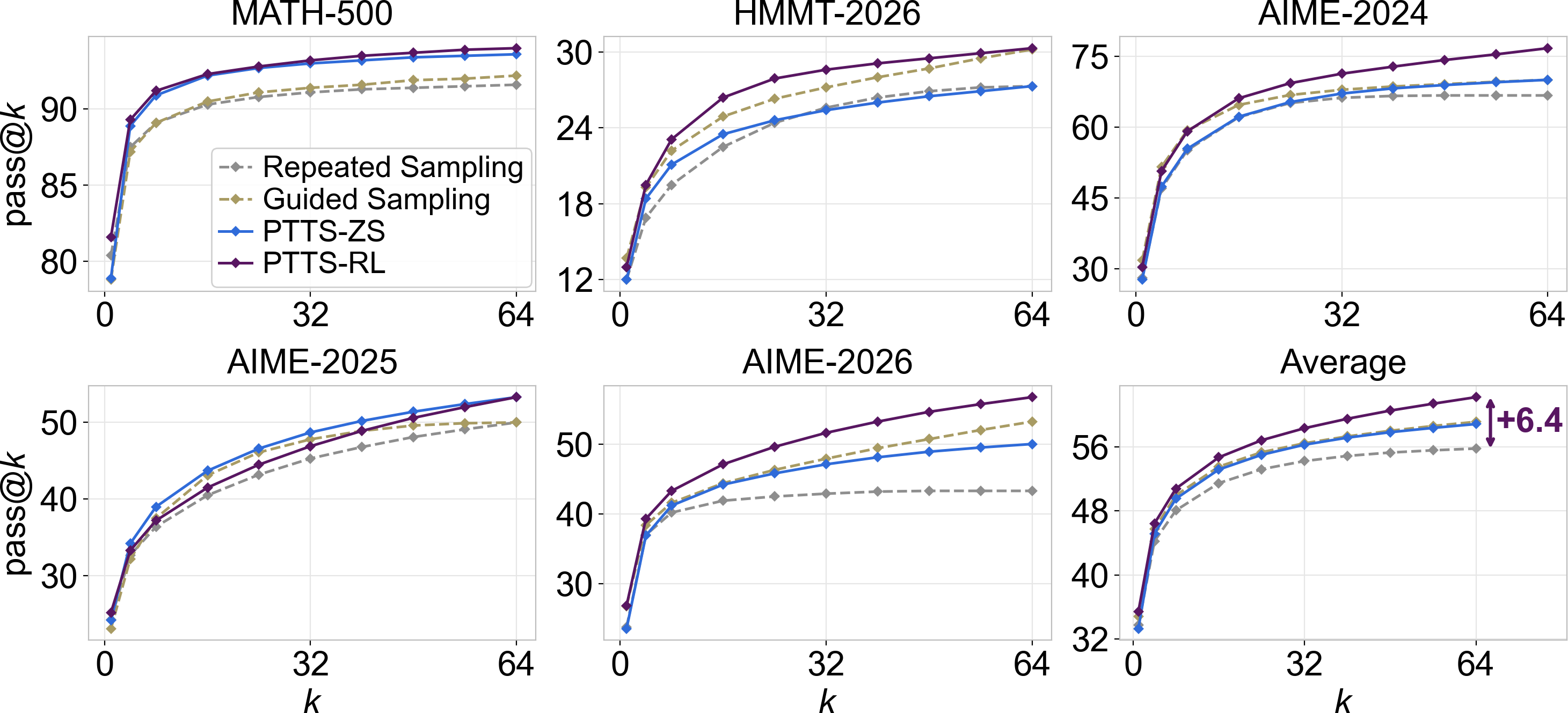}
        \vspace{-1.8em}
        \caption{\small \textbf{Main results for Qwen3-1.7B.} PTTS-RL improves \passof{64} over repeated sampling by up to 13.4 (6.4 on average), while its \passof{32} already exceeds repeated sampling's \passof{64} (58.3 vs. 55.8).}
        \label{fig:main_1.7b}
\vspace{-1.6em}
\end{figure}

\mypara{Training settings.} We train the planner from the base models using the DAPO training set \citep{yu2026dapo}, optimizing for the \passof{k} ($k=4$) reward. As in \S\ref{sec:method:method}, we use a truncated execution budget of $4k$ tokens for each solution. We train the \texttt{Qwen3-1.7B/4B-Base} models with a learning rate of 1e-6 for 240 steps and use the final checkpoint; detailed hyperparameter settings are provided in Appendix \ref{app:training-details}.

\mypara{Scaling PTTS to larger test-time budgets.} In both inference (PTTS-ZS and PTTS-RL) and training (PTTS-RL), we use a branching factor of $k=4$. To evaluate larger total budgets $K>k$, we independently repeat the PTTS procedure until we collect $K$ executor solutions in total, and then compute \passof{K} over the aggregated set. For PTTS-RL, this allows us to keep the same planner while seamlessly scaling to larger inference budgets. As discussed in \S\ref{sec:exp:ablations}, further increasing $k$ yields diminishing returns; therefore, we use $k=4$ as a computationally efficient setting to demonstrate our method.

\subsection{Main Results and Discussions}
\label{sec:exp:main}

\mypara[]{Main results.}
Figures~\ref{fig:main_1.7b} and~\ref{fig:main_4b} present our main results for the 1.7B and 4B models, respectively. Despite its simple design, PTTS-ZS consistently outperforms repeated sampling, with \passof{64} gains of up to 6.7 for 1.7B models and 3.4 for 4B models. PTTS-RL further improves upon PTTS-ZS and consistently outperforms both repeated sampling and Guided Sampling. Specifically, compared to repeated sampling, PTTS-RL improves \passof{64} by up to 13.4 for 1.7B models and 6.7 for 4B models. Beyond performance gains, improved compute efficiency is a broader benefit of PTTS. Most notably, at both model scales, PTTS-RL's \passof{32} exceeds repeated sampling's \passof{64}, achieving better performance with only half the sampling budget. Overall, these results demonstrate the benefit of coordinating test-time compute across complementary reasoning paths.

\mypara{Diversity as an indicator of $\mathbf{pass@k}$.}
To better understand what drives these gains, we examine whether broader coverage of reasoning strategies is associated with better \passof{k} performance. Concretely, we prompt GPT-5-mini to group outlines and solutions into clusters and report the number of unique clusters as the diversity metric, as detailed in Appendix \ref{app:concept-clustering}.
We then use Qwen3-1.7B's PTTS-ZS outputs on AIME-2024--2026 to analyze the rank-averaged correlations among outline diversity, solution diversity, and \passof{4} across groups of $k=4$ reasoning branches.
Results show that greater outline diversity is associated with higher solution diversity (Spearman's $\rho = 0.57$) and better resulting \passof{4} ($\rho = 0.59$), while solution diversity is also positively correlated with \passof{4} ($\rho = 0.32$).
Overall, these results support diversity as a meaningful indicator of \passof{k} performance across both the outline and solution levels.

\begin{figure}[!tbh]
    \centering
        \centering
        \includegraphics[width=\linewidth]{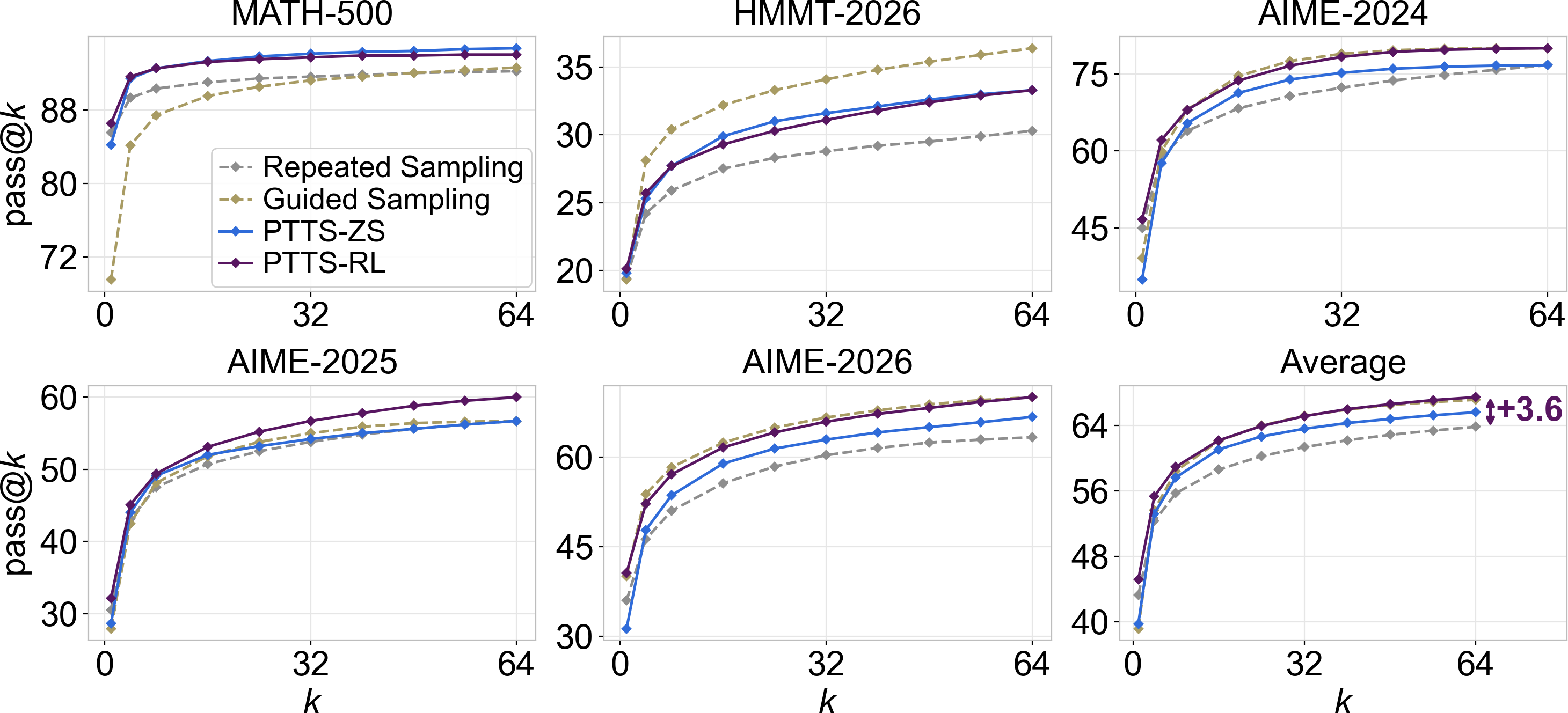}
        \vspace{-1.8em}
        \caption{\small \textbf{Main results for Qwen3-4B.} PTTS-RL improves \passof{64} over repeated sampling by up to 6.7 (3.6 on average), while its \passof{32} already exceeds repeated sampling's \passof{64} (65.1 vs. 63.8).}
        \label{fig:main_4b}
\vspace{-.5em}
\end{figure}

\begin{figure}[!tbhp]
    \centering
    \includegraphics[width=\linewidth]{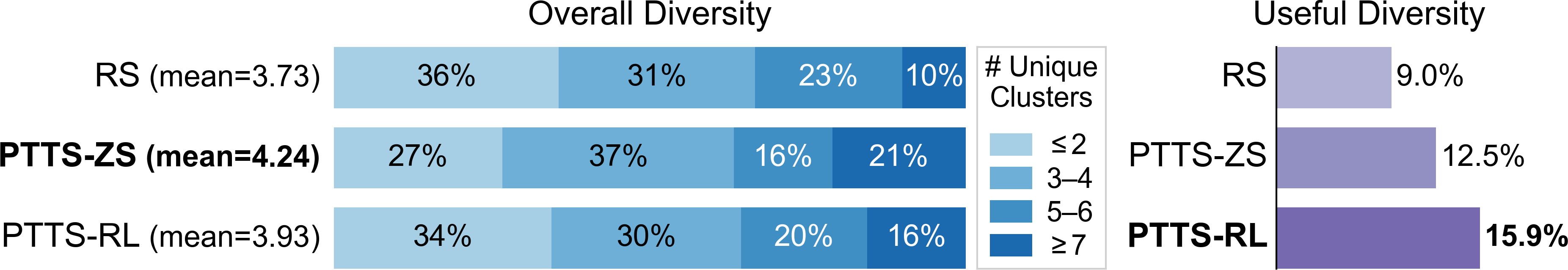}
    \vspace{-1em}
    \caption{\small\textbf{Solution diversity for repeated sampling (RS), PTTS-ZS, and PTTS-RL}, measured on Qwen3-1.7B outputs for AIME 2024--2026. \textbf{Left:} Overall diversity, measured by the number of distinct solution clusters among $k\!=\!64$ branches (Appendix \ref{app:concept-clustering}). Bars show the fraction of problems in each diversity range, with parentheses indicating the mean number of distinct clusters. \textbf{Right:} Useful diversity, measured as the fraction of correct solutions belonging to distinct solution clusters (Appendix \ref{app:diversity-metrics}).}
    \vspace{-1em}
    \label{fig:avg_diversity}
\end{figure}

\mypara{Diversity induced by PTTS.} We next investigate whether PTTS effectively enhances diversity. Using the same diversity measure discussed above, we compare the diversity of repeated sampling, PTTS-ZS, and PTTS-RL on Qwen3-1.7B outputs on the AIME-2024--2026 datasets.
As shown in Figure~\ref{fig:avg_diversity}, both PTTS variants increase overall response diversity over repeated sampling, with PTTS-ZS achieving the highest diversity across generated solutions.
Optimized for the \passof{k} objective, PTTS-RL does not further increase overall diversity over PTTS-ZS, but instead steers diversity toward branches that reach correct answers.
Measuring this \textit{useful diversity} as the normalized count of distinct clusters among correct branches (detailed in Appendix~\ref{app:diversity-metrics}), PTTS-RL achieves the highest \textit{useful diversity} of 15.9\%, outperforming both repeated sampling and PTTS-ZS.
Figure~\ref{fig:qualitative_diversity} further shows a qualitative example of the increased diversity under PTTS.
Overall, PTTS not only increases reasoning diversity but also steers it toward successful reasoning paths,  thereby improving \passof{k}.

\mypara{Outline adherence of the executor.} Beyond the planner, PTTS also relies on the executor to faithfully develop each proposed outline into a full solution. Reasoning models, however, are not explicitly trained for outline-guided generation, making it unclear how closely their solutions adhere to the provided outlines.
We evaluate outline adherence via LLM-as-a-judge using GPT-5-mini, as detailed in Appendix~\ref{app:outline-validity-following}. Analysis of PTTS-ZS and PTTS-RL outputs on the AIME-2024--2026 datasets shows that Qwen3-1.7B and Qwen3-4B follow outlines reasonably well, achieving adherence rates of 70\% and 71\%, respectively.
Interestingly, qualitative observations show that while executors typically begin by either strictly following or reiterating the outline verbatim, they may pivot to alternative strategies if the initial approach fails, with an example shown in Figure \ref{fig:qualitative_adherence_drift}. This may dilute the outline's influence on final performance as reasoning traces become longer, thereby affecting the training signal, as discussed in \S\ref{sec:exp:ablations}.

\subsection{Ablation Studies}
\label{sec:exp:ablations}

\mypara[]{Ablations on the branching factor $k$.}
We study the effect of the PTTS branching factor $k$ while fixing the total inference budget at $K=64$, following the scaling procedure in \S\ref{sec:exp:settings}.
As shown in Figure~\ref{fig:qwen3_k_ablation_pass64}, moving from repeated sampling ($k=1$) to modest branching factors substantially improves PTTS-ZS, with performance peaking at $k=4$ for Qwen3-1.7B and $k=8$ for Qwen3-4B. Beyond these points, increasing $k$ yields no consistent \passof{64} gains, with performance fluctuating at larger values. Since larger $k$ provides no reliable benefit while substantially increasing PTTS-RL training cost, we use $k=4$ throughout our main experiments.

\begin{figure}[!tbhp]
    \centering

    \begin{minipage}[t]{0.62\textwidth}
        \vspace{0pt}
        \centering
        \includegraphics[width=\linewidth]{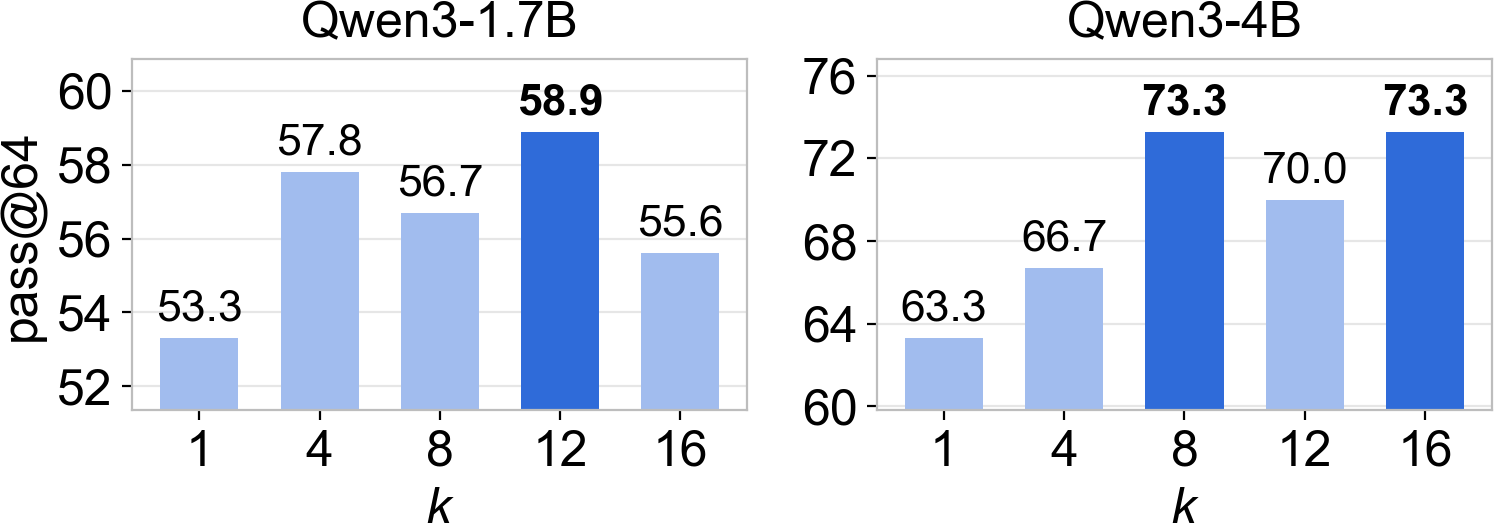}
        \vspace{-1.6em}
        \caption{\small \Passof{64} performance of PTTS-ZS across branching factors $k$ at a fixed inference budget of $K=64$, averaged across AIME 2024--2026; $k=1$ denotes repeated sampling.}
        \label{fig:qwen3_k_ablation_pass64}
    \end{minipage}
    \hfill
    \begin{minipage}[t]{0.35\textwidth}
        \vspace{0pt}
        \centering
        \includegraphics[width=\linewidth]{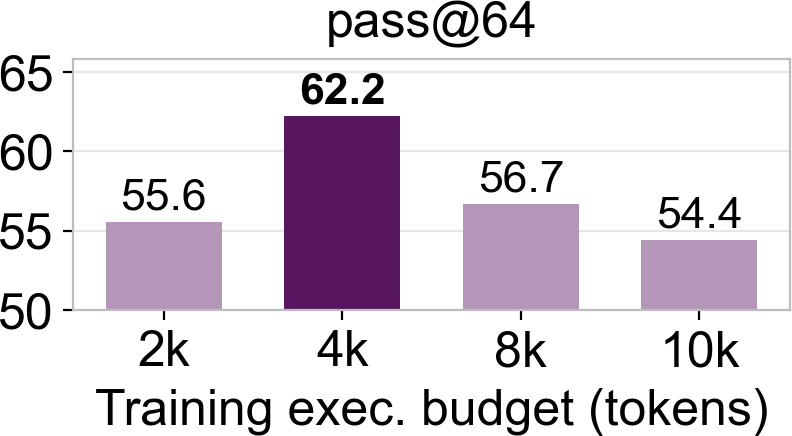}
        \vspace{-1.6em}
        \caption{\small \Passof{64} performance of PTTS-RL using Qwen3-1.7B trained with varying execution budgets, averaged across AIME 2024--2026.}
        \label{fig:ablation_trunc_exec}
    \end{minipage}
    \vspace{-1.2em}
\end{figure}

\mypara{Ablations on truncated execution.} To evaluate the impact of training-time execution limits, we train Qwen3-1.7B with PTTS-RL using execution budgets from $2k$ to $10k$ tokens, while fixing the inference budget at $10k$. As shown in Figure~\ref{fig:ablation_trunc_exec}, performance peaks at $4k$ and degrades with larger budgets: matching the $10k$ inference budget lowers average \passof{64} by 7.8 points and throughput by 27\%. Longer executions allow the executor to recover from poor outlines through extended rethinking, causing these outlines to receive positive reward and weakening the training signal. Conversely, an overly strict $2k$ budget prematurely truncates valid reasoning, making the reward 
\begin{wraptable}{r}{0.44\textwidth}
      \centering
      \small
      \begin{tabular}{l@{\hspace{8pt}}l|rrrrrr}
          \hline
  \multirow{2}{*}{Algorithm} & \multirow{2}{*}{Planner} & \multicolumn{3}{c}{\Passof{k}} \\
           & & $k\!\!=\!\!1$ & $k\!\!=\!\!4$ & $k\!\!=\!\!64$ \\
          \hline
  \multicolumn{3}{l}{\textit{Qwen3-1.7B as executor}} & \\
  \hline
          RS & - & 25.4 & 38.9 & 53.3 \\
          PTTS-ZS & 1.7B  & 25.2 & 39.5 & 57.8 \\
       \rowcolor{gray!12}   PTTS-ZS & 4B  & 25.6 & 40.4 & 60.0 \\
          PTTS-RL & 1.7B  & \textbf{27.5} & 41.1 & \textbf{62.2} \\
      \rowcolor{gray!12}    PTTS-RL & 4B  & \textbf{27.5} & \textbf{41.3} & 61.1 \\
          \hline
        \multicolumn{3}{l}{\textit{Qwen3-4B as executor}} & \\
  \hline
          RS & - &  37.2 & 49.4 & 65.6 \\
       \rowcolor{gray!12}      PTTS-ZS & 1.7B  & 36.8 & 52.0 & 68.9 \\
          PTTS-ZS & 4B   & 31.6 & 49.8 & 66.7 \\
      \rowcolor{gray!12}      PTTS-RL & 1.7B  & \textbf{42.1} & \textbf{54.7} & \textbf{70.0} \\
       PTTS-RL & 4B & 39.8 & 53.1 & \textbf{70.0} \\
          \hline
      \end{tabular}
      \vspace{-.5em}
      \caption{\small \Passof{k} performance of repeated sampling (RS), PTTS-ZS, and PTTS-RL under different planner and executor configurations.
{\sethlcolor{gray!12}\hl{Highlighted rows}} indicate configurations where the planner and executor differ in size.}
      \label{tab:diff_planner_executor}
      \vspace{-2em}
\end{wraptable}  uninformative and substantially hurting performance. Overall, a $4k$ budget yields the optimal balance, providing enough length to reliably execute outlines and actively truncating unguided rethinking.

\mypara{Ablations on planner and executor sizes.}
While our main experiments use planners and executors of the same sizes, decoupling them reveals distinct scaling behaviors. As shown in Table \ref{tab:diff_planner_executor}, 1.7B planners effectively guide a larger Qwen3-4B executor, yielding significant gains over repeated sampling and matching the performance of 4B planners, despite a train-test mismatch: the 1.7B PTTS-RL planner is trained with the smaller Qwen3-1.7B executor. In contrast, using 4B planners to guide a Qwen3-1.7B executor provides no additional benefit over 1.7B planners. These results suggest a promising strategy of training small but capable planners alongside small executors, then scaling only the executor at inference time.

\begin{wrapfigure}{r}{0.33\textwidth}
    \centering
    \vspace{-1.5em}
    \includegraphics[width=\linewidth]{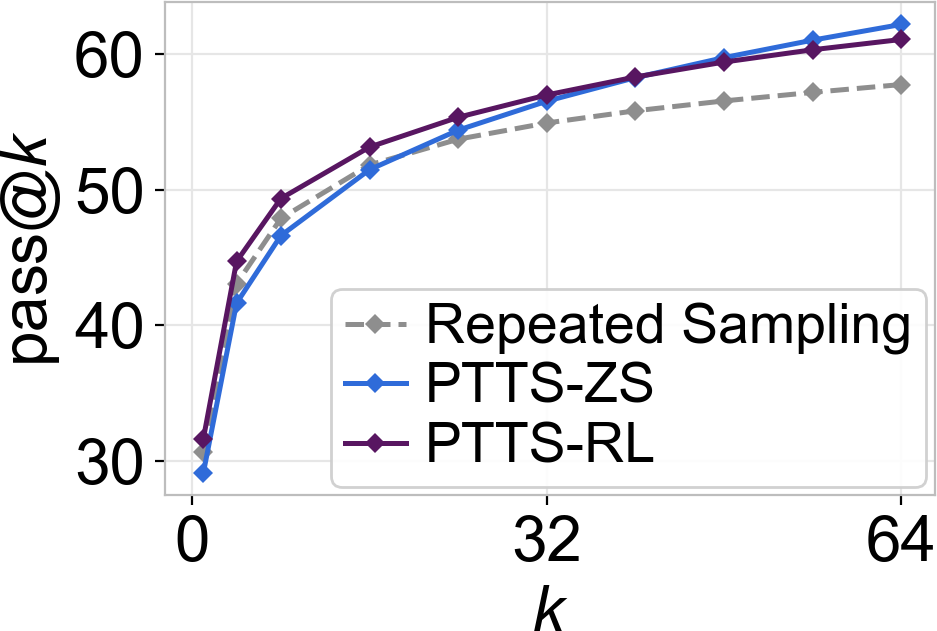}
    \vspace{-1.7em}
    \caption{\small \Passof{k}  of repeated sampling, PTTS-ZS, and PTTS-RL with \texttt{e3-1.7B} as the executor.}
    \label{fig:e3_1.7b}
    \vspace{-1.5em}
\end{wrapfigure}
\mypara{Combination with diversity-aware RL.} As discussed in \S\ref{sec:related_work}, PTTS operates at inference time and is complementary to training-side diversity-aware improvements. We demonstrate this by using e3 \citep{setlur2026e3} as the executor, which is trained to improve exploration through a structured curriculum. With \texttt{e3-1.7B} as the executor, both the zero-shot 1.7B planner and the RL planner trained with a Qwen3-1.7B executor outperform repeated sampling by 4.4 and 3.3 points, respectively. These gains show that PTTS remains effective on top of a diversity-aware RL executor, providing an additional and complementary source of improvement.

\section{Conclusion}

We introduce Planned Test-Time Scaling (PTTS), a framework that addresses the redundant reasoning and repeated errors of standard repeated sampling by explicitly coordinating attempts across $k$ different branches. PTTS decomposes test-time scaling under a budget of $k$ reasoning branches into a planner stage that produces diverse solution outlines and an executor stage that follows each outline. We instantiate PTTS with a zero-shot variant, PTTS-ZS, which plans using an off-the-shelf model, and a trained variant, PTTS-RL, which directly optimizes the planner for the pass@$k$ reward. We show theoretically and empirically that PTTS outperforms independent repeated sampling: across five mathematical datasets, both PTTS-ZS and PTTS-RL consistently yield gains of up to 13.4 points. Overall, PTTS provides a principled paradigm for test-time scaling, highlighting explicit coordination and budget allocation as key to better unlocking its potential.

\bibliography{iclr2026_conference}
\bibliographystyle{iclr2026_conference}

\appendix

\section{Proofs of Propositions}
\label{app:math}

\newtheorem{lemma}[proposition]{Lemma}   %
\newenvironment{proof}[1][Proof]{\noindent\textit{#1.}~}{\hfill$\square$\par}  %

As discussed in \S\ref{sec:method:advantage}, we consider a finite space for data category $\mathcal{C}$ and a finite space for reasoning mode $\mathcal{Z}$. For a problem $\x{}$ with data category $c(\x{})=c\in\mathcal{C}$, a policy $\pi(\cdot\mid\x{})$ chooses a reasoning mode $z\sim\pi(\cdot\mid\x{}),~z\in\mathcal{Z}$ for solving this task, and the resulting success rate $R_{zc}$ is determined only by the selected mode $z$ and data category $c$; when multiple attempts are issued on the same problem, their outcomes are independent conditional on the data category, each succeeding with probability $R_{zc}$ under its selected mode.

Critically, we assume an \textbf{imperfect policy} that cannot perfectly observe the true data category $c(\x{})$ and perform optimal mode selection accordingly. Instead, we assume the policy that chooses mode based on observed data category $c'(\x{})$ with imperfect data categorization,
\begin{align*}
\pi(\cdot\mid\x{}) = \pi(\cdot\mid c'(\x{})),
\end{align*}
with some emission probability conditioned on the true data category $c(\x{})$, i.e.
\begin{align*}
\Pr\left[ c'(\x{}) = c' \mid c(\x{}) = c \right].
\end{align*}
For notation convenience, we assume this probability is strictly positive for any $(c,c')$ pair; similarly, we assume the marginal probability $\Pr[c(\x{}) = c]$ is strictly positive for all $c$. Then we have a posterior distribution of
\begin{align*}
\omega(c\mid c') \triangleq \Pr \left[ c(\x{}) = c \mid c'(\x{}) = c'\right]
\end{align*}
that is also non-zero for any $(c,c')$. Then, for a certain observed data category $c'$, we have the \textit{posterior-averaged success rate} of each mode $z \in \mathcal{Z}$ under the observation $c'$,
\begin{align*}
\bar{R}_{zc'} \;\triangleq\; \mathbb{E}_{c\sim\omega(\cdot\mid c')}\left[ R_{zc} \right] \;=\; \sum_{c\in\mathcal{C}} \omega(c\mid c')\, R_{zc},
\end{align*}
i.e. the success rate that mode $z$ effectively attains through the imperfect observation: since the true category is unobserved, a single attempt with mode $z$ succeeds with probability exactly $\bar{R}_{zc'}$ given $c'(\x{}) = c'$ by the tower rule. Based on this formulation, we enforce two critical assumptions that will drive the rest of the proof:

\begin{assumption}[Mode dominance]\label{assump:mode_dominance}
The observed category $c'\in\mathcal{C}$ is dominated by a single mode $z^*\in\mathcal{Z}$:
\begin{align*}
\forall z \in \mZ{}, z \neq z^*,~\bar{R}_{zc'} < \bar{R}_{z^*c'}.
\end{align*}
\end{assumption}

\begin{assumption}[Complementary modes]\label{assump:complementary}
There exists a single worst category $c^\dagger\in\mathcal{C}$,
\begin{align*}
\exists ~c^\dagger \in \mathcal{C}, ~\forall c \in \mathcal{C},c \neq c^{\dagger}, ~ R_{z^*c} > R_{z^*c^\dagger},
\end{align*}
and the category benefits more from a mode other than $z^*$:
\begin{align*}
\exists z \in \mZ{}, ~ R_{zc^\dagger} > R_{z^*c^\dagger}.
\end{align*}
\end{assumption}

We now consider two types of policies, a \textbf{GRPO-RS} policy and a \textbf{PTTS-RL} policy, both operating on the observed category $c'(\x{})=c'$ and issuing $k$ attempts under inference-time budget $k$. Each attempt will adopt a reasoning mode $z$, resulting in a solution \y{} with a binary outcome $r(\y{};\x{}) \in \{0,1\}$, where $\Pr[r(\y{};\x{})= 1] = R_{zc}$.

\mypara{GRPO-RS policy}, as the standard approach, optimizes a marginal policy against \passof{1}, and utilizes repeated sampling (RS) to scale to a larger inference-time budget $k$. With slight abuse of notation, we write $\passof{1}(\pi)$ for the expectation of \passof{1} under policy $\pi$. Formally, the marginal policy is trained to maximize \passof{1}, i.e. the single-attempt success rate,
\begin{align*}
\pi_{\mathrm{rs}}(\cdot\mid c') \;\in\; \argmax_{\pi}\; \passof{1}(\pi),
\end{align*}
where
\begin{align*}
\passof{1}(\pi) &\triangleq \mathbb{E}_{c\sim\omega(\cdot\mid c'),\, z\sim\pi(\cdot\mid c')} \left[ \mathbf{1}\left(r(\y{};\x{})= 1\right) \right] = \mathbb{E}_{c\sim\omega(\cdot\mid c')} \left[ \sum_{z\in\mathcal{Z}} \pi(z\mid c')\, R_{zc} \right] \\
&= \sum_{z\in\mathcal{Z}} \pi(z\mid c')\, \bar{R}_{zc'}.
\end{align*}
At inference budget $k$, RS draws modes i.i.d. from this \emph{fixed} marginal, $z_1,\dots,z_k \overset{\text{i.i.d.}}{\sim} \pi_{\mathrm{rs}}(\cdot\mid c')$, and runs one attempt per draw. The resulting \passof{k} has expectation
\begin{align*}
\passof{k}\left(\piof{rs}\right) = 1 - \mathbb{E}_{c\sim\omega(\cdot\mid c')} \left[ \Big( \sum_{z\in\mathcal{Z}} \pi_{\mathrm{rs}}(z\mid c')\, (1-R_{zc}) \Big)^{k} \right],
\end{align*}
with residual error $\epsilon^{\mathrm{rs}}_k \;\triangleq\; 1 - \passof{k}\left(\piof{rs}\right)$.

\mypara{PTTS-RL policy}, on the other hand, optimizes the use of the entire inference budget directly against \passof{k}. At inference budget $k$, the policy selects an allocation $\mathbf{k} = (k_z)_{z\in\mathcal{Z}}$ that assigns the $k$ attempt slots across reasoning modes, with $k_z \in \mathbb{Z}_{\ge 0}$ and $\sum_{z\in\mathcal{Z}} k_z = k$. An allocation $\mathbf{k}$ then achieves a \passof{k} of:
\begin{align*}
\passof{k}(\mathbf{k}) = 1 - \mathbb{E}_{c\sim\omega(\cdot\mid c')} \left[ \prod_{z\in\mathcal{Z}} (1-R_{zc})^{k_z} \right],
\end{align*}
and the PTTS-RL policy at budget $k$ selects the optimal allocation
\begin{align*}
\mathbf{k}^{(k)} \;\in\; \argmax_{\mathbf{k}:\,\sum_{z} k_z = k}\; \passof{k}(\mathbf{k}),
\end{align*}
with residual error $\epsilon^{\mathrm{ptts}}_k \;\triangleq\; 1 - \passof{k}\bigl(\mathbf{k}^{(k)}\bigr)$ and induced mode-selection distribution $\pi^{(k)}_{\mathrm{ptts}}(z\mid c') \triangleq k^{(k)}_z / k$. In contrast to GRPO-RS whose marginal distribution is optimized once against \passof{1}, PTTS-RL is optimized against \passof{k} specifically for the budget $k$, and different attempts within a single budget may use different modes.

\mypara{Result 1: GRPO-RS collapses onto the dominant mode.} Optimizing against \passof{1} alone leaves no reason to preserve mode diversity: under mode dominance, the optimal marginal policy is unique and deterministic.

\begin{proposition}[GRPO-RS selects only $z^*$]\label{prop:rs_collapse}
Under Assumption~\ref{assump:mode_dominance}, the unique \passof{1}-optimal marginal policy for the observed category $c'$ is the point mass on the dominant mode, $\pi_{\mathrm{rs}}(\cdot\mid c') = \delta_{z^*}$. Consequently, at every inference budget $k$, GRPO-RS issues all $k$ attempts with mode $z^*$, and its residual error is
\begin{align*}
\epsilon^{\mathrm{rs}}_k = \sum_{c\in\mathcal{C}} \omega(c\mid c')\, (1-R_{z^*c})^{k}.
\end{align*}
\end{proposition}

\begin{proof}
Since $\passof{1}(\pi) = \sum_{z} \pi(z\mid c')\, \bar{R}_{zc'}$ is linear in $\pi$ and $\sum_{z} \pi(z\mid c') = 1$,
\begin{align*}
\passof{1}(\delta_{z^*}) - \passof{1}(\pi) = \sum_{z \neq z^*} \pi(z\mid c') \left( \bar{R}_{z^*c'} - \bar{R}_{zc'} \right),
\end{align*}
which is non-negative by Assumption~\ref{assump:mode_dominance} and zero if and only if $\pi(z\mid c') = 0$ for every $z \neq z^*$; hence $\delta_{z^*}$ is the unique maximizer. GRPO-RS therefore issues every attempt with $z^*$, and substituting $\pi_{\mathrm{rs}} = \delta_{z^*}$ into $\passof{k}\left(\piof{rs}\right)$ yields the closed form.
\end{proof}

\mypara{Result 2: PTTS-RL diversifies beyond a finite budget.} Optimizing against \passof{k} directly leads to a qualitatively different behavior: beyond an explicit finite budget, the optimal allocation never concentrates on $z^*$ alone. Let $z^\dagger \in \mathcal{Z}$ denote a complementary mode furnished by Assumption~\ref{assump:complementary}, i.e. $R_{z^\dagger c^\dagger} > R_{z^*c^\dagger}$ (note this forces $z^\dagger \neq z^*$), and define the complementarity gap on the worst category, the failure rate of $z^*$ on its worst category, and its largest failure rate on any other category:
\begin{align*}
\Delta \;\triangleq\; R_{z^\dagger c^\dagger} - R_{z^*c^\dagger} \;>\; 0,
\qquad
q^\dagger \;\triangleq\; 1 - R_{z^*c^\dagger},
\qquad
\rho \;\triangleq\; \max_{c \neq c^\dagger}\, (1 - R_{z^*c}).
\end{align*}

\begin{proposition}[PTTS-RL assigns mass outside $z^*$]\label{prop:diversify}
Under Assumptions~\ref{assump:mode_dominance} and \ref{assump:complementary}, define the finite threshold
\begin{align*}
k^*_{\mathrm{ptts}} \;\triangleq\; 2 + \left\lfloor \frac{\ln \left( \omega(c^\dagger\mid c')\, \Delta \right)}{\ln \left( \rho / q^\dagger \right)} \right\rfloor ,
\end{align*}
with the convention $k^*_{\mathrm{ptts}} \triangleq 2$ when $\rho = 0$. Then for every budget $k \geq k^*_{\mathrm{ptts}}$, every optimal allocation $\mathbf{k}^{(k)}$ satisfies $\sum_{z \neq z^*} k^{(k)}_z \geq 1$; equivalently, the induced policy $\pi^{(k)}_{\mathrm{ptts}}$ assigns non-zero probability mass to modes other than $z^*$.
\end{proposition}

\begin{proof}
The threshold is well-defined and finite. First, $|\mathcal{C}| \geq 2$: if $\mathcal{C} = \{c^\dagger\}$, mode dominance would give $R_{z^\dagger c^\dagger} = \bar{R}_{z^\dagger c'} < \bar{R}_{z^*c'} = R_{z^*c^\dagger}$, contradicting Assumption~\ref{assump:complementary}; hence $\rho$ is a maximum over a non-empty finite set. Next, $q^\dagger \geq \Delta > 0$ (as $R_{z^\dagger c^\dagger} \leq 1$), and $0 \leq \rho < q^\dagger$ since $c^\dagger$ is the single worst category of $z^*$. Finally, $\omega(c^\dagger\mid c') \in (0,1)$ because the posterior is strictly positive on at least two categories, so $\omega(c^\dagger\mid c')\,\Delta \in (0,1)$. For $\rho > 0$, both logarithms are therefore negative and $k^*_{\mathrm{ptts}} \geq 2$ is finite.

Now fix any $k \geq 2$ and compare the pure allocation $\mathbf{k}^{\mathrm{all}}$, placing all $k$ slots on $z^*$, against the swapped allocation $\mathbf{k}^{\mathrm{swap}}$, placing $k-1$ slots on $z^*$ and one on $z^\dagger$. Since the corresponding failure products differ only in one factor,
\begin{align*}
\passof{k}(\mathbf{k}^{\mathrm{swap}}) - \passof{k}(\mathbf{k}^{\mathrm{all}})
&= \sum_{c\in\mathcal{C}} \omega(c\mid c')\, (1-R_{z^*c})^{k-1} \left( R_{z^\dagger c} - R_{z^*c} \right) \\
&\geq\; \omega(c^\dagger\mid c')\, \Delta\, (q^\dagger)^{k-1} - \rho^{k-1},
\end{align*}
where the last step keeps the $c^\dagger$ term exactly and bounds every other term via $|R_{z^\dagger c} - R_{z^*c}| \leq 1$, $1-R_{z^*c} \leq \rho$, and $\sum_{c \neq c^\dagger} \omega(c\mid c') \leq 1$. If $\rho = 0$, the right-hand side is positive for every $k \geq 2 = k^*_{\mathrm{ptts}}$. Otherwise it is positive iff $(\rho/q^\dagger)^{k-1} < \omega(c^\dagger\mid c')\,\Delta$, i.e. iff $k - 1 > \ln(\omega(c^\dagger\mid c')\Delta) / \ln(\rho/q^\dagger)$, since dividing by $\ln(\rho/q^\dagger) < 0$ flips the inequality; this holds for every $k \geq k^*_{\mathrm{ptts}}$ because $k^*_{\mathrm{ptts}} - 1 = 1 + \lfloor \cdot \rfloor$ strictly exceeds the ratio. Hence the pure allocation is strictly suboptimal for every $k \geq k^*_{\mathrm{ptts}}$, and since only finitely many allocations of size $k$ exist, every optimal allocation places at least one slot on a mode other than $z^*$.
\end{proof}

Note that $k^*_{\mathrm{ptts}}$ is a sufficient budget rather than an exact transition point—the proof discards every favorable term outside $c^\dagger$—and it decreases as the worst category becomes more probable (larger $\omega(c^\dagger\mid c')$) or more fixable (larger $\Delta$): hedging pays off sooner when the failure mode of $z^*$ is both likely and repairable.

\mypara{Result 3: PTTS-RL reduces error exponentially and strictly faster; GRPO-RS may not.}
Having characterized how each policy allocates its attempts, we now compare their residual errors and the exponential rates at which these decay. The two policies part ways on both counts: PTTS-RL drives its error to zero exponentially fast under Assumptions~\ref{assump:mode_dominance} and \ref{assump:complementary} alone, and with a strictly better error exponent, whereas GRPO-RS converges only when its dominant mode retains some success probability on the worst category. Throughout, let
\begin{align*}
r \;\triangleq\; \frac{1-R_{z^\dagger c^\dagger}}{q^\dagger} \;\in\; [0,1),
\end{align*}
which is well-defined since $q^\dagger > 0$ and below one since $R_{z^\dagger c^\dagger} > R_{z^*c^\dagger}$ by Assumption~\ref{assump:complementary}.

\begin{proposition}[Error rates and strict separation]\label{prop:rates}
Under Assumptions~\ref{assump:mode_dominance} and \ref{assump:complementary}:
\begin{enumerate}
    \item[(i)] $\epsilon^{\mathrm{ptts}}_k \leq \epsilon^{\mathrm{rs}}_k$ for every budget $k \geq 1$.
    \item[(ii)] $\omega(c^\dagger\mid c')\, (q^\dagger)^{k} \leq \epsilon^{\mathrm{rs}}_k \leq (q^\dagger)^{k}$ for every $k \geq 1$, and hence
    \begin{align*}
    \lambda_{\mathrm{rs}} \;\triangleq\; \lim_{k\to\infty} \bigl(\epsilon^{\mathrm{rs}}_k\bigr)^{1/k} \;=\; q^\dagger .
    \end{align*}
    In particular, $\epsilon^{\mathrm{rs}}_k \to 0$ if and only if $R_{z^*c^\dagger} > 0$; otherwise $\epsilon^{\mathrm{rs}}_k$ converges to the plateau $\omega(c^\dagger\mid c') > 0$: GRPO-RS never solves problems from the worst category.
    \item[(iii)] For every $\alpha \in (0,1)$ satisfying $\rho^{1-\alpha} < q^\dagger$, which exists since $\rho < q^\dagger$,
    \begin{align*}
    \limsup_{k\to\infty}\, \bigl(\epsilon^{\mathrm{ptts}}_k\bigr)^{1/k} \;\leq\; \max\left\{ q^\dagger r^{\alpha},\;\; \rho^{\,1-\alpha} \right\} \;<\; q^\dagger \;=\; \lambda_{\mathrm{rs}} .
    \end{align*}
\end{enumerate}
In particular, $\epsilon^{\mathrm{ptts}}_k \to 0$ at least exponentially fast even when $\epsilon^{\mathrm{rs}}_k$ does not vanish, and PTTS-RL strictly improves the error exponent over GRPO-RS.
\end{proposition}

\begin{proof}
\emph{(i)} The pure allocation placing all $k$ slots on $z^*$ is feasible for the optimization defining $\mathbf{k}^{(k)}$, and its \passof{k} coincides with that of GRPO-RS by Proposition~\ref{prop:rs_collapse}; optimality of $\mathbf{k}^{(k)}$ yields the claim.

\emph{(ii)} Start from the closed form of Proposition~\ref{prop:rs_collapse}. Since $c^\dagger$ is the single worst category of $z^*$ by Assumption~\ref{assump:complementary}, every category satisfies $1-R_{z^*c} \leq q^\dagger$; bounding every term accordingly gives the upper bound, and retaining only the $c^\dagger$ term gives the lower bound, where $\omega(c^\dagger\mid c') > 0$ since the posterior is non-zero for every pair. Both bounds hold for every value of $R_{z^*c^\dagger}$, including $q^\dagger = 1$; taking $k$-th roots and using $\omega(c^\dagger\mid c')^{1/k} \to 1$ gives $\lambda_{\mathrm{rs}} = q^\dagger$. If $R_{z^*c^\dagger} > 0$, then $q^\dagger < 1$ and the upper bound vanishes. If instead $R_{z^*c^\dagger} = 0$, the $c^\dagger$ term is frozen at $\omega(c^\dagger\mid c')$, while every remaining term vanishes as $k \to \infty$ since $1-R_{z^*c} \leq \rho < q^\dagger = 1$; hence $\epsilon^{\mathrm{rs}}_k$ converges to $\omega(c^\dagger\mid c')$.

\emph{(iii)} Fix $\alpha \in (0,1)$ with $\rho^{1-\alpha} < q^\dagger$; such $\alpha$ exists because $\rho^{1-\alpha} \to \rho < q^\dagger$ as $\alpha \to 0^+$. For every $k > \frac{1}{1-\alpha}$, so that $\lceil \alpha k \rceil < k$, consider the feasible allocation placing $\lceil \alpha k \rceil$ slots on $z^\dagger$ and the remaining $k - \lceil \alpha k \rceil$ slots on $z^*$. Its $c^\dagger$ term satisfies, using $r < 1$ and $\lceil \alpha k \rceil \geq \alpha k$,
\begin{align*}
\omega(c^\dagger\mid c')\,(q^\dagger)^{\,k - \lceil \alpha k \rceil}\, (1-R_{z^\dagger c^\dagger})^{\lceil \alpha k \rceil}
\;\leq\; (q^\dagger)^{k}\, r^{\lceil \alpha k \rceil}
\;\leq\; \left( q^\dagger r^{\alpha} \right)^{k} .
\end{align*}
Every other term vanishes if $\rho = 0$; otherwise, using $1-R_{z^\dagger c} \leq 1$, $1-R_{z^*c} \leq \rho < 1$, $\lceil \alpha k \rceil \leq \alpha k + 1$, and $\sum_{c \neq c^\dagger} \omega(c\mid c') \leq 1$, these terms contribute at most
\begin{align*}
\rho^{\,k - \lceil \alpha k \rceil}
\;\leq\; \rho^{\,(1-\alpha)k - 1}
\;=\; \rho^{-1} \left( \rho^{\,1-\alpha} \right)^{k} .
\end{align*}
Writing $\mu \triangleq \max\{ q^\dagger r^{\alpha},\, \rho^{1-\alpha} \}$, optimality of $\mathbf{k}^{(k)}$ over feasible allocations therefore certifies
\begin{align*}
\epsilon^{\mathrm{ptts}}_k \;\leq\; \left( 1 + \rho^{-1} \right) \mu^{k} \quad \text{when } \rho > 0,
\qquad
\epsilon^{\mathrm{ptts}}_k \;\leq\; \mu^{k} \quad \text{when } \rho = 0.
\end{align*}
Taking $k$-th roots gives $\limsup_k (\epsilon^{\mathrm{ptts}}_k)^{1/k} \leq \mu$. It remains to check $\mu < q^\dagger$: the first branch satisfies $q^\dagger r^{\alpha} < q^\dagger$ since $r < 1$ and $\alpha > 0$, and the second satisfies $\rho^{1-\alpha} < q^\dagger$ by the choice of $\alpha$. Since $\mu < q^\dagger \leq 1$, the certificate also shows $\epsilon^{\mathrm{ptts}}_k \to 0$ at least exponentially fast.
\end{proof}

The separation quantifies the value of hedging: GRPO-RS decays exactly at the rate at which its single mode clears the worst category, and whenever $z^*$ is hopeless on that category, repeated sampling replays the same failure and the error plateaus at exactly its posterior mass. Devoting even a constant fraction of the budget to a complementary mode strictly improves the error exponent, with the residual error dominated by whichever category the mixed allocation covers worst.

\subsection{Closed-Form Analysis of a Synthetic Example}
\label{app:math:eg}

Finally, we apply our framework to a two-category, two-mode example to derive analytical expressions for both $\passof{k}$ and mode coverage. These results are visualized in Figure~\ref{fig:mode_collapse_visualization} of the main text.

\mypara{Instantiation.} Let $\mathcal{C} = \{c_1, c_2\}$ and $\mathcal{Z} = \{z_1, z_2\}$, with a prior probability of $\Pr[c(\x{}) = c_1] = p = 0.75$. We consider a ``blind'' policy as an extreme case of an imperfect policy, where the observed category $c'(\x{})$ is independent of the true category. Consequently, the posterior reduces to the prior: $\omega(c_1\mid c') = p = 0.75$ and $\omega(c_2\mid c') = 1-p = 0.25$. We define the success rates as follows:
\begin{align*}
R_{z_1c_1} = R_{z_2c_2} = a = 0.40, \quad
R_{z_1c_2} = R_{z_2c_1} = b = 0.02.
\end{align*}
Thus, each mode is highly effective on exactly one category.

This instantiation satisfies the two critical assumptions. The posterior success rates are $\bar{R}_{z_1c'} = pa + (1-p)b = 0.305$ for $z_1$, and $\bar{R}_{z_2c'} = pb + (1-p)a = 0.115$ for $z_2$; as $\bar{R}_{z_1c'} > \bar{R}_{z_2c'}$, Assumption~\ref{assump:mode_dominance} holds with the dominant mode $z^* = z_1$. For $z_1$, then, $c^\dagger = c_2$ is the single worst category that is better served by $z^\dagger = z_2 \neq z_1$. Thus, Assumption~\ref{assump:complementary} also holds.

\mypara{GRPO-RS in closed form.} By Proposition~\ref{prop:rs_collapse}, GRPO-RS always selects the dominant mode $z_1$ (see Figure~\ref{fig:mode_collapse_visualization:2}), yielding:
\begin{align*}
\passof{k}\left(\piof{rs}\right)&= 1 - p\cdot (1-a)^{k} - (1-p)\cdot (1-b)^{k} \\
&= 1 - 0.75 \times 0.6^{k} - 0.25 \times 0.98^{k}.
\end{align*}
As visualized in Figure~\ref{fig:mode_collapse_visualization:1}, the performance curve quickly saturates on $c_1$ problems (as the $0.6^k$ term decays rapidly to near zero) and then only slowly improves on the remaining $c_2$ problems (due to the slow decay of the $0.98^k$ term).

\mypara{PTTS-RL in closed form.} By Proposition~\ref{prop:diversify}, PTTS-RL splits the budget across both modes once $k$ is large enough. In this symmetric instance, the optimal allocation places $m(k) \approx k/2 - 1.12$ (rounded to an integer) slots on $z_2$, which becomes non-zero and breaks the mode collapse at $k = 4$, as visualized in Figure~\ref{fig:mode_collapse_visualization:2}. The \passof{k} is then derived as:
\begin{align*}
\passof{k}\bigl(\mathbf{k}^{(k)}\bigr)
&= 1 - p\,(1-a)^{\left(k-m(k)\right)}(1-b)^{m(k)} - (1-p)\,(1-b)^{\left(k-m(k)\right)}(1-a)^{m(k)} \\
&= 1 - 0.75\times 0.6^{\left(k-m(k)\right)}\times 0.98^{m(k)} - 0.25\times 0.98^{\left(k-m(k)\right)}\times 0.6^{m(k)},
\end{align*}
which strictly exceeds $\passof{k}\left(\piof{rs}\right)$ once $m(k) > 0$. Moreover, since PTTS-RL splits the inference budget approximately evenly ($m(k) \sim k/2$), the residual error decays at a rate of approximately $\sqrt{(1-a)(1-b)} = \sqrt{0.6 \times 0.98} \approx 0.77$ per attempt, significantly faster than the $0.98^{k}$ term of GRPO-RS, leading to the wide gap  in Figure~\ref{fig:mode_collapse_visualization:1}.

\section{Training Details}
\label{app:training-details}

For the PTTS-RL training, we initialize from the Qwen3-1.7B/4B-Base models, using the corresponding Qwen3-1.7B or Qwen3-4B models as executors. The execution budget is fixed as $4k$ tokens, and the branching factor $k$ is set as 4. Our implementation is based on \texttt{verl} \citep{sheng2024hybridflow}. We train on the DAPO-Math-17K dataset \citep{yu2026dapo} using 8x 80GB NVIDIA A100 GPUs. Each model is trained for 240 optimization steps, and we report results from the final checkpoint. Additional hyperparameters are shown in Table~\ref{tab:hyperparameters}.

\begin{table}[h]
\centering
\begin{tabular}{ll}
\hline
\textbf{Hyperparameter} & \textbf{Value} \\
\hline
Batch size & 16 \\
Responses per prompt & 16 \\
PPO clipping ratio (lower) & 0.20 \\
PPO clipping ratio (upper) & 0.28 \\
Learning rate & $1\times10^{-6}$ \\
Rollout temperature & 1.0 \\
Weight decay & 0.1 \\
Gradient clipping & 1.0 \\
\hline
\end{tabular}
\caption{Training hyperparameters.}
\label{tab:hyperparameters}
\end{table}

\section{Analysis Details}

\subsection{Outline Adherence Analysis}
\label{app:outline-validity-following}

We use GPT-5-mini to evaluate how faithfully the executor follows the outline generated by the planner. Specifically, for each instance, we provide the judge with the problem statement, the planner-generated outline, and the corresponding executor-generated response, and ask it to assess whether the executor's primary reasoning path implements the strategy, decomposition, or perspective specified by the outline. The detailed prompt is shown in Prompt \ref{prompt:outline-following}. The model assigns a score on a five-point scale, which we normalize to [0, 1] for ease of interpretation and comparison across settings.

\subsection{Diversity Analysis by Clustering}
\label{app:concept-clustering}

We quantify the diversity of both outlines and solutions by clustering them into distinct conceptual groups and computing the metrics reported in Appendix \ref{app:diversity-metrics}. We apply the same procedure separately to planner outlines and executor solutions.

Concretely, the clustering procedure consists of two steps. (1) \textbf{Concept extraction}: we prompt GPT-5-mini to identify the single most important mathematical or logical concept underlying each planner outline $o_i$ and each executor response $\y_i$. The extracted concept may be a theorem, formula, invariant, transformation, or canonical solution strategy. The full prompt is shown in Prompt~\ref{prompt:concept-extraction}. (2) \textbf{Concept clustering}: after extracting the concepts independently, we provide GPT-5-mini with all extracted concepts for each problem and ask it to group concepts that represent the same underlying mathematical mechanism. The full prompt is shown in Prompt~\ref{prompt:concept-clustering}.

\subsection{Diversity Metrics}
\label{app:diversity-metrics}

For a given problem, consider $k$ responses, either $k$ outlines sampled from the planner or $k$ solutions generated by repeated sampling or PTTS. We use the procedure described above in Appendix \ref{app:concept-clustering} to assign each response $i$ to a cluster $c_i$, and use the following metrics to quantify their diversity:

\mypara{Overall diversity.}
We measure overall diversity by the number of distinct clusters represented among the $k$ responses, formally defined as:
\begin{align*}
\mathcal{C} = \{c_i \mid i=1,\ldots,k\},
\qquad
\textmd{Overall diversity} = |\mathcal{C}|.
\end{align*}
Intuitively, a larger value indicates that the responses cover a broader range of distinct reasoning approaches.

\mypara{Useful diversity.}
Overall diversity can be easily hacked by producing responses that are superficially diverse but ultimately unhelpful for solving the problem. We therefore also measure diversity among responses that eventually lead to a correct answer. Let $r_i \in \{0,1\}$ indicate whether response $i$ leads to a correct solution, and define
\begin{align*}
\mathcal{C}^{+} = \{c_i \mid r_i = 1\}.
\end{align*}
We normalize the number of distinct clusters among correct responses by the total number of correct responses:
\begin{align*}
\textmd{Useful diversity}
= |\mathcal{C}^{+}| ~/~ \sum_{i=1}^{k} r_i.
\end{align*}
Thus, useful diversity is high when correct responses span distinct reasoning approaches, and low when many correct responses concentrate on the same approach.

\subsection{Ranked-Averaged Correlation Test}
\label{app:diversity:correlation_by_rank}

When reporting an aggregated correlation between two metrics, such as diversity and \passof{k}, we have $M$ groups of responses for each problem ($M=16$ in \S\ref{sec:exp:main}), each with a diversity score and a \passof{k} value. Typically, each problem provides too few points to reliably estimate a correlation on its own. We therefore leverage data across multiple problems to identify a consistent trend. However, points from different problems cannot be pooled directly: simply concatenating them introduces problem difficulty as a confounding factor, while directly averaging across problems is not meaningful because there is no natural correspondence between groups from different problems.

We therefore use a rank-averaging approach that aligns points across problems by their \textit{relative position} rather than their raw values. For a variable $a$ (e.g., outline diversity), let $a_{q,i}$, $i=1,\ldots,M$, denote its values for problem $q$, and let $b_{q,i}$ denote the corresponding values of a target variable $b$ (e.g., \passof{k}). For each problem, we sort the $M$ groups by $a$. Let $a_{q,(j)}$ denote the $j$-th smallest value of $a$, and let $b_{q,(j)}$ denote the $b$-value of the same group. We then average both variables across problems at each rank $j$:
\begin{align*}
\bar a_j = \frac{1}{Q}\sum_{q=1}^{Q} a_{q,(j)}, \qquad
\bar b_j = \frac{1}{Q}\sum_{q=1}^{Q} b_{q,(j)}.
\end{align*}
Finally, we compute the Spearman correlation between $\{\bar a_j\}_{j=1}^M$ and $\{\bar b_j\}_{j=1}^M$.

We apply this procedure to three pairs in \S\ref{sec:exp:main}: outline diversity vs.\ response diversity, outline diversity vs.\ \passof{k}, and response diversity vs.\ \passof{k}.

\clearpage
\section{Qualitative Examples}

This section presents additional qualitative examples, including an example of diversity induced by PTTS in Figure~\ref{fig:qualitative_diversity} and an example of outline adherence and drift in Figure~\ref{fig:qualitative_adherence_drift}.

\definecolor{ptthlcolor}{RGB}{0,86,160}
\newcommand{\ptthl}[1]{{\color{ptthlcolor}\bfseries #1}}
\definecolor{pttdvcolor}{RGB}{204,85,0}
\newcommand{\pttdv}[1]{{\color{pttdvcolor}\bfseries #1}}
\newcommand{\pttbox}[1]{%
  \noindent\fcolorbox{black!35}{black!3}{%
    \parbox{\dimexpr\linewidth-2\fboxsep-2\fboxrule\relax}{%
      \setlength{\parskip}{4pt}\setlength{\parindent}{0pt}#1}}%
}
\newcommand{\pttok}{{\color{green!45!black}\checkmark}}
\newcommand{\pttno}{{\color{red!70!black}$\times$}}
\newcommand{\pttdash}{\leaders\hrule height 2.7pt depth -2.3pt\hfill}
\newcommand{\pttsep}{\par\vspace{-1pt}\noindent\textcolor{black!30}{\rule{\linewidth}{0.4pt}}\par\vspace{-1pt}}
\newcommand{\pttbr}[1]{\par\vspace{1pt}\noindent
  \hbox to \linewidth{\color{black!40}\pttdash\hspace{6pt}%
    \textcolor{black}{\textbf{Branch #1}}\hspace{6pt}\pttdash}\par\vspace{-1pt}}

\begin{figure}[!hp]
\scriptsize
\setlength{\parindent}{0pt}

\pttbox{%
\textbf{Problem:} Let $\triangle ABC$ be a right triangle with $\angle A = 90^\circ$ and $BC = 38.$ There exist points $K$ and $L$ inside the triangle such that $AK = AL = BK = CL = KL = 14.$ The area of the quadrilateral $BKLC$ can be expressed as $n\sqrt3$ for some positive integer $n.$ Find $n.$

\textbf{Answer:} $104$}

\vspace{5pt}
\pttbox{%
\textbf{Repeated Sampling:} All four branches develop the same line of reasoning: they place $A$ at the origin with $B$ and $C$ on the axes, derive $AB^2 + AC^2 = 1444$, and locate $K$ and $L$ as intersections of equal-radius circles. None resolves the resulting system. Across the full pool of $64$ branches, only two distinct solution clusters appear.

\pttsep
$\mathbf{Pass@4 = 0}$}

\vspace{6pt}
\pttbox{%
\textbf{PTTS-RL}
\pttbr{1}
\textbf{Outline:} \ptthl{Symmetry and Equilateral Triangles}: Since $AK = AL = BK = CL = 14$, we can consider that points $K$ and $L$ form two equilateral triangles with sides of length 14. \ldots\ we can infer that the quadrilateral $BKLC$ is composed of two congruent triangles \ldots

\textbf{Reasoning:} First, the problem mentions that $AK = AL = BK = CL = 14$. So, points $K$ and $L$ are each 14 units away from $A$ and from $B$ and $C$ respectively. Also, $KL$ is 14. \ptthl{The solution outline suggests symmetry and equilateral triangles.} Let me try to visualize this. \ldots\ \ptthl{But the outline says that quadrilateral $BKLC$ is composed of two congruent triangles.} Maybe triangles $BKC$ and $LKC$? Wait, not sure. Alternatively, maybe triangles $BKL$ and $LKC$? Hmm. Alternatively, \ptthl{since $AKL$ is equilateral, and $BK = CL = 14$, maybe there's some symmetry here.} \ldots

\textbf{Final answer:} none. \quad \textbf{Correct:} \pttno

\pttbr{2}
\textbf{Outline:} \ptthl{Coordinate Geometry}: We can place the triangle $\triangle ABC$ in a coordinate plane, with point $A$ at the origin $(0,0)$, point $B$ on the x-axis, and point $C$ on the y-axis. \ldots\ calculate the area of quadrilateral $BKLC$ using the Shoelace formula or the determinant method \ldots

\textbf{Reasoning:} \ptthl{First, let me recall the solution outline. The problem suggests using coordinate geometry.} So I should place triangle $ABC$ on a coordinate plane. Let me try to visualize that. \ptthl{Since angle $A$ is $90^\circ$, I can put point $A$ at $(0,0)$, point $B$ on the x-axis, so $B$ is $(b, 0)$, and point $C$ on the y-axis, so $C$ is $(0, c)$.} Then, $BC$ is the hypotenuse, so the distance between $B$ and $C$ is $\sqrt{b^2 + c^2} = 38$. So that gives me the equation $b^2 + c^2 = 38^2 = 1444$. \ldots\ \ptthl{So now, coordinates of $K$ are $(b/2, \sqrt{196 - b^2/4})$ and coordinates of $L$ are $(\sqrt{196 - c^2/4}, c/2)$.} Now, we need to use the $KL$ distance equation \ldots

\textbf{Final answer:} none. \quad \textbf{Correct:} \pttno

\pttbr{3}
\textbf{Outline:} \ptthl{Trigonometric Relationships and Angles}: \ldots\ if we denote the angle $\angle BAK$ as $\theta$, we can use trigonometric identities to express the lengths of the segments $BK$ and $CL$ in terms of $\theta$. \ldots

\textbf{Reasoning:} Alternatively, maybe I can consider triangles $ABK$ and $ALC$. Since $AK = BK = 14$, triangle $ABK$ is isoceles with $AK = BK$. \ldots\ \ptthl{Therefore, angle at $A$ is angle $BAK$, which I denoted earlier as $\theta$. Then, using the law of cosines:} \ldots\ Simplify: $196 = c^2 + 196 - 28c \cos(\theta) \Rightarrow 0 = c^2 - 28c \cos(\theta) \Rightarrow c^2 = 28c \cos(\theta) \Rightarrow c = 28 \cos(\theta)$. Therefore, $\cos(\theta) = c / 28$. \ldots\ \ptthl{Therefore, coordinates of $K$ are $(c/2, 14 \sin(\theta)) = (14 \cos(\theta), 14 \sin(\theta))$} \ldots

\textbf{Final answer:} none. \quad \textbf{Correct:} \pttno

\pttbr{4}
\textbf{Outline:} \ptthl{Geometric Transformations and Similarity}: \ldots\ if we rotate point $K$ around point $A$ by $120^\circ$ (a characteristic angle in equilateral triangles), we can create a new point that might help in understanding the relationship between $K$ and $L$. \ldots

\textbf{Reasoning:} \ptthl{The solution outline suggests using geometric transformations, like rotations. Maybe a rotation by 120 degrees?} Since 120 degrees is related to equilateral triangles, which have angles of 60 degrees, and maybe that could help with the distances here. Let me think. \ldots\ \ptthl{Alternatively, maybe there's a rotation that maps $K$ to $L$ or something like that. \ldots\ Suppose we rotate point $K$ around $A$ by 60 degrees. If that rotation maps $K$ to $L$, then $AK = AL = 14$, and angle between $AK$ and $AL$ is 60 degrees. Then triangle $AKL$ would be equilateral, which matches the given $KL = 14$.} So that seems possible. \ldots\ But regardless, the area of $BKLC$ came out to be $104\sqrt{3}$. Given that the steps are correct and the logic holds, I think that is the answer. Therefore, the value of $n$ is 104.

\textbf{Final answer:} $104$. \quad \textbf{Correct:} \pttok

\pttsep
$\mathbf{Pass@4 = 1}$}

\vspace{-0.5em}
\caption{\textbf{Qualitative example of diversity induced by PTTS}, generated by PTTS-RL with Qwen3-4B. Reasoning parts that explicitly follow the given outline are \ptthl{highlighted}. The four reasoning traces cover distinct directions, with one reaching the correct final answer.}
\label{fig:qualitative_diversity}
\end{figure}

\begin{figure}[!ht]
\scriptsize
\setlength{\parindent}{0pt}

\pttbox{%
\textbf{Problem:} Find the sum of all real numbers $r$ such that there is at least one point where the circle with radius $r$ centered at $(4, 39)$ is tangent to the parabola with equation $2y = x^2 - 8x + 12.$

\textbf{Answer:} $50$}

\vspace{6pt}
\pttbox{%
\textbf{Outline:} \ptthl{Geometric Distance Approach}: Begin by considering the geometric interpretation of tangency between a circle and a parabola. \ptthl{The distance from the center of the circle $(4, 39)$ to any point on the parabola must equal the radius $r$} of the circle at the point of tangency. \ptthl{Define the distance formula} between the center of the circle and a general point $(x, y)$ on the parabola. This distance will be expressed in terms of $x$ and $y$, with $y$ being determined by the parabola's equation. Since the circle and parabola are tangent, \ptthl{the distance equation will have exactly one solution for $x$}. This condition will provide a relationship between $x$, $y$, and $r$. \ptthl{Use the discriminant of the resulting quadratic equation in $x$} to ensure that there is exactly one solution, which implies that \ptthl{the discriminant must be zero}. This will give us a condition on $r$ without needing to solve for $x$ explicitly. Solve the condition on $r$ obtained from the discriminant to find the possible values of $r$. The sum of these values will be the answer.

\pttsep
\textbf{Reasoning:} First, I need to recall what it means for a circle and a parabola to be tangent. Tangent means they touch at exactly one point. So, the distance from the center of the circle to the parabola must be equal to the radius $r$, and there should be exactly one such point.

\ptthl{The solution outline suggests using the geometric distance approach. So, I should start by considering a general point $(x, y)$ on the parabola.} Since the parabola's equation is given as $2y = x^2 - 8x + 12$, I can express $y$ in terms of $x$. \ldots\ Dividing both sides by 2: $y = (x^2 - 8x + 12)/2$. So, any point on the parabola can be represented as $(x, (x^2 - 8x + 12)/2)$.

\ptthl{Now, the distance from the center of the circle $(4, 39)$ to this point $(x, y)$ should be equal to the radius $r$. The distance formula is $\sqrt{(x - 4)^2 + (y - 39)^2} = r$.} But since $y$ is expressed in terms of $x$, I can substitute that into the distance formula. \ldots\ To eliminate the square root, I can square both sides: $(x - 4)^2 + \left[\left((x^2 - 8x + 12)/2 - 39\right)\right]^2 = r^2$.

\ptthl{Now, this equation should have exactly one solution for $x$ because the circle and parabola are tangent at that point. Therefore, the equation must have exactly one real root. To find the condition for this, I need to simplify the equation and then find the discriminant of the resulting quadratic (or higher degree) equation in $x$.} \ldots\ So, rearranged: $x^4 - 16x^3 - 64x^2 + 1024x + (4420 - 4r^2) = 0$.

Now, this is a quartic equation in $x$. \ldots\ \ptthl{However, the solution outline says that the equation must have exactly one solution for $x$, which would correspond to the case where they are tangent.} However, perhaps the equation is a quartic, which can have multiple roots. But for tangency, maybe there is a multiple root? So, the quartic equation should have a multiple root, which would mean that the equation and its derivative both equal zero at that root. \pttdv{But the solution outline says that the discriminant of the resulting quadratic equation in $x$ must be zero. Wait, but the equation we have is quartic. Hmm, maybe I made a mistake here. Let me check again.} \ldots\ But that seems correct. However, the problem mentions that the discriminant of the resulting quadratic equation in $x$ must be zero. But here, we have a quartic equation. So maybe there's a miscalculation here. \pttdv{Alternatively, perhaps there is a different approach.} \ldots\ But how do I find the values of $r$ for which the quartic equation has at least one multiple root? \ldots\ Alternatively, maybe the equation we derived (the quartic) can be factored. Let me check if that's possible. \ldots\ But this seems complicated. Maybe there's a better way. Let me think again.

\pttdv{Alternatively, perhaps I can use calculus. If the circle and parabola are tangent, then they have a common point and their tangent lines at that point are the same. So, maybe I can set up the system of equations for the circle and the parabola, and then impose that their derivatives are equal at the point of intersection. Let me try that approach.} \ldots\ $dy/dx = [2x - 8]/2 = x - 4$. Now, find the derivative of the circle implicitly. \ldots\ $dy/dx = -(x - 4)/(y - 39)$. At the point of tangency, the derivatives must be equal, so: $x - 4 = -(x - 4)/(y - 39)$. Assuming that $x \neq 4$ (if $x = 4$, then the derivative of the parabola is 0, and we can check separately). So, if $x \neq 4$, then we can divide both sides by $(x - 4)$: $1 = -1/(y - 39)$. Therefore, $y - 39 = -1 \Rightarrow y = 38$. \ldots\ But $x$ is $4 \pm 4\sqrt5$, so $(x - 4)^2$ is $(\pm 4\sqrt5)^2 = 16 \cdot 5 = 80$. Therefore, the distance is $\sqrt{80 + 1} = \sqrt{81} = 9$. Therefore, $r = 9$. \ldots\ If $x = 4$, then from the parabola equation, $y = (16 - 32 + 12)/2 = (-4)/2 = -2$. So, the point is $(4, -2)$. The distance from $(4, 39)$ to $(4, -2)$ is $\sqrt{0 + (39 - (-2))^2} = \sqrt{41^2} = 41$. Therefore, $r = 41$. \ldots\ Therefore, there are two values of $r$: 9 and 41. Their sum is $9 + 41 = 50$. So, the answer should be 50.

\pttdv{But wait, why did the earlier approach with the quartic equation not lead me to this? Because when I derived the quartic equation, I considered that for tangency, the equation has exactly one solution. However, in reality, there are two different radii: one corresponding to the point $(4, -2)$ and another corresponding to the points $(4 \pm 4\sqrt5, 38)$.} \ldots

\textbf{Final answer:} $50$. \quad \textbf{Correct:} \pttok}

\vspace{-0.5em}
\caption{\textbf{Qualitative example of outline adherence and drift}, from PTTS-RL with Qwen3-4B. Reasoning parts that explicitly follow the outline are marked in \ptthl{blue}, while parts that deviate are marked in \pttdv{orange}. The executor initially follows the outline but deviates when the outlined approach fails, eventually reaching the correct answer. Such recovery can assign positive reward to a poor outline, weakening the training signal, as discussed in \S\ref{sec:method:method}.}
\label{fig:qualitative_adherence_drift}
\end{figure}

\section{Prompts}

\renewcommand{\lstlistingname}{Prompt}

This section presents the prompts used for planning, execution, and analysis, including the PTTS planner prompt in Prompt~\ref{prompt:planner}, the executor prompt in Prompt~\ref{prompt:executor}, the outline adherence evaluation prompt in Prompt~\ref{prompt:outline-following}, and the concept extraction and clustering prompts in Prompts~\ref{prompt:concept-extraction} and~\ref{prompt:concept-clustering}.

\begin{lstlisting}[
    style=promptstyle,
    caption={Prompt used by the PTTS planner.},
    captionpos=b,
    label={prompt:planner}
]
[SYSTEM]

You are an annotator tasked with generating multiple high-level solution outlines for a math problem.

Your goal is to explore different perspectives, strategies, or conceptual approaches that could be used to solve the problem. Based on this, produce {num_suboutlines} distinct outlines that could independently guide a solver from start to finish.
* You must NOT solve the problem.
* You must NOT compute values, simplify expressions, or use algebra.
* Any calculation makes the output invalid.

Format output as a numbered list (1. - {num_suboutlines}.), where each item is an outline.

[USER]

{question}
\end{lstlisting}

\begin{lstlisting}[
    style=promptstyle,
    caption={Prompt used by the PTTS executor.},
    captionpos=b,
    label={prompt:executor}
]
[SYSTEM]

You are a math problem solver. You are given a math problem and a solution outline. Follow the outline carefully to solve the problem step by step. Show your work and put your final answer in \boxed{}.

[USER]

Problem:
{question}

Solution Outline:
{outline}

Now solve the problem by following this outline:
\end{lstlisting}

\begin{lstlisting}[
    style=promptstyle,
    caption={Prompt used to evaluate outline following.},
    captionpos=b,
    label={prompt:outline-following}
]
[SYSTEM]

You are an expert evaluator for mathematical reasoning and
outline-conditioned generation.

You will be given:
1. A math problem.
2. An outline provided to a model.
3. A solution response generated by the model conditioned on that outline.

Determine whether the solution response follows the provided outline.

Evaluation rules:
- Do not assign a high following score merely because the solution is correct.
- Do not assign a low following score merely because the solution is incorrect.
- A response follows an outline when its primary reasoning path uses the strategy, decomposition, or perspective described by the outline.
- A response may introduce additional details while still following the outline, provided that those details are consistent with the outlined approach.
- A response that uses a substantially different method should receive a low following score, even if it reaches the correct answer.
- If the outline or response is too unclear, malformed, or irrelevant to assess, use CANNOT_JUDGE.

Outline following rubric:
5 = FOLLOWS_EXACTLY: Clearly follows the intended strategy.
4 = MOSTLY_FOLLOWS: Mainly follows the outline but adds, skips, or changes minor steps.
3 = PARTIALLY_FOLLOWS: Uses some ideas from the outline but substantially deviates.
2 = DOES_NOT_FOLLOW: Uses a different strategy or largely ignores the outline.
1 = CANNOT_JUDGE: The outline or response is too unclear, malformed, or irrelevant to assess.

Return only valid JSON using the following schema:

{
  "outline_following_score": <integer from 1 to 5>,
  "outline_following_label":
    "<FOLLOWS_EXACTLY | MOSTLY_FOLLOWS | PARTIALLY_FOLLOWS |
      DOES_NOT_FOLLOW | CANNOT_JUDGE>",
  "outline_following_note": "<brief explanation>",
}

[USER]

[QUESTION]
{question}
[/QUESTION]

[OUTLINE]
{outline}
[/OUTLINE]

[OUTLINE_CONDITIONED_SOLUTION]
{response}
[/OUTLINE_CONDITIONED_SOLUTION]

Evaluate the outline-following behavior according to the system instructions. Return only valid JSON.
\end{lstlisting}

\begin{lstlisting}[
    style=promptstyle,
    caption={Prompt used to extract the pivotal concept from each executor response.},
    captionpos=b,
    label={prompt:concept-extraction}
]
You are given a math problem and a full solution response or outline generated by a model.

The response may be correct or incorrect.

Your task is to identify the SINGLE most important mathematical or logical concept, theorem, canonical formula, invariant, transformation, or strategy that makes the attempted solution or outline possible.

Important:
- Summarize the concept actually used or attempted in the response.
- Do NOT repair the response into a better concept.
- Do NOT introduce a concept that is not supported by the response.
- If the response is flawed, still identify the key concept the response attempted to use.
- If the response uses multiple concepts, choose the one without which the attempted solution would not work, usually the first pivotal step.
- Choose the narrowest concept that still covers the attempted solution.
  - Good: "Pythagorean Theorem", "Vieta's Formulas", "Modular Invariant", "Complement Counting".
  - Bad: "Geometry", "Algebra", "Number Theory", "Counting".

Definition of a valid key concept:
- It should be a specific mathematical concept, theorem, formula, invariant, transformation, or canonical strategy.
- It should explain the core mechanism of the attempted solution.
- It should not be a full outline.
- It should not include the final answer.
- It should not include detailed computations, equation chains, or numerical derivations.
- It should be specific enough to distinguish this solution path from other possible approaches.

Formatting rules:
- The concept name should use Title Case and singular form when possible.
- Prefer standard mathematical names when available.
- If no standard theorem applies, use a concise strategy label, such as "Casework on Remainders", "Symmetry Reduction", "Complement Counting", or "Bounding Argument".
- Do not use overly broad labels such as "Algebra", "Geometry", "Combinatorics", or "Number Theory" unless the response is too unclear to identify a narrower concept.

Return ONLY valid JSON in the following format:
{{
  "concept": "<single most important concept name>",
  "evidence": "<one short sentence explaining why this concept is pivotal in the attempted response>",
  "confidence": <integer from 1 to 5>
}}

Confidence score:
5 = clearly identifiable and specific concept.
4 = mostly clear concept, with minor ambiguity.
3 = plausible concept, but response uses several competing ideas or is partially unclear.
2 = weak guess because the response is flawed, vague, or inconsistent.
1 = cannot meaningfully identify a concept.

[QUESTION]
{question}
[/QUESTION]

[FULL_SOLUTION_RESPONSE]
{response}
[/FULL_SOLUTION_RESPONSE]
\end{lstlisting}

\begin{lstlisting}[
    style=promptstyle,
    caption={Prompt used to cluster response concepts for each problem.},
    captionpos=b,
    label={prompt:concept-clustering}
]
You are given a math problem and a list of candidate key concepts extracted from model-generated solution responses or outlines.

Some source responses may be correct and some may be incorrect.

Your task is to cluster the candidate concepts by the underlying mathematical or logical concept used in the attempted solution.

Clustering rules:
- Group concepts together if they refer to essentially the same pivotal mathematical idea, theorem, invariant, transformation, or canonical strategy.
- Merge synonymous or near-synonymous names.
  - Example: "Modulo Invariant", "Invariant Modulo 5", and "Residue Class Invariant" may belong together if they describe the same attempted idea.
- Separate concepts if they represent meaningfully different solution mechanisms, even if they are from the same broad field.
  - Example: "Vieta's Formulas" and "Discriminant Condition" should usually be separated.
  - Example: "Complement Counting" and "Inclusion-Exclusion Principle" should be separated unless the evidence shows they refer to the same pivotal step.
- Cluster by attempted concept, not by whether the source response was correct.
- A cluster may contain concepts extracted from both correct and incorrect source responses.
- Do NOT repair an incorrect or vague concept into a better one.
- Do NOT introduce a new concept that is not represented by the cluster members.
- Prefer the narrowest canonical concept name that covers the cluster.

For each cluster, write:
- "canonical_concept": the standard or clearest concept name.
- "member_indices": the indices of candidate concepts in this cluster.
- "canonical_description": one short sentence describing the shared concept.

Canonical concept rules:
- Use Title Case and singular form when possible.
- Prefer standard mathematical names when available.
- If no standard theorem applies, use a concise strategy label, such as "Casework on Remainders", "Symmetry Reduction", "Complement Counting", or "Bounding Argument".
- Do not use overly broad labels such as "Algebra", "Geometry", "Combinatorics", or "Number Theory" unless the cluster is genuinely too vague to identify a narrower concept.
- Do not include final answers, computations, equation chains, or problem-specific numerical values.

Return ONLY valid JSON in this exact schema:
{{
  "clusters": [
    {{
      "cluster_id": 1,
      "canonical_concept": "<Title Case concept name>",
      "member_indices": [1, 3],
      "canonical_description": "<one short sentence describing the shared concept>"
    }}
  ]
}}

Example:

[EXAMPLE_QUESTION]
Find the number of positive integers n satisfying a certain divisibility condition.
[/EXAMPLE_QUESTION]

[EXAMPLE_CANDIDATE_CONCEPTS]
1. Concept: Modular Arithmetic Filtering
   Evidence: The solution restricts possible values of n using congruences modulo small integers.
2. Concept: Casework on Parity
   Evidence: The solution splits n into odd and even cases and eliminates impossible cases.
3. Concept: Congruence Class Filtering
   Evidence: The solution uses residues modulo small primes to narrow the candidate values.
4. Concept: Exhaustive Enumeration
   Evidence: The solution checks all feasible candidates directly against the condition.
[/EXAMPLE_CANDIDATE_CONCEPTS]

[EXAMPLE_EXPECTED_OUTPUT]
{{
  "clusters": [
    {{
      "cluster_id": 1,
      "canonical_concept": "Modular Arithmetic Filtering",
      "member_indices": [1, 3],
      "canonical_description": "Use congruence constraints to narrow the possible values before checking the original condition."
    }},
    {{
      "cluster_id": 2,
      "canonical_concept": "Casework On Parity",
      "member_indices": [2],
      "canonical_description": "Split the variable into parity cases and analyze each case separately."
    }},
    {{
      "cluster_id": 3,
      "canonical_concept": "Exhaustive Enumeration",
      "member_indices": [4],
      "canonical_description": "Check all feasible candidates directly against the required condition."
    }}
  ]
}}
[/EXAMPLE_EXPECTED_OUTPUT]

Now cluster the actual candidate concepts below.

[QUESTION]
{question}
[/QUESTION]

[CANDIDATE_CONCEPTS]
{candidate_concepts}
[/CANDIDATE_CONCEPTS]
\end{lstlisting}

\section{AI Assistants In Research Or Writing}

We used AI assistants solely for stylistic improvements in writing, such as improving clarity, grammar, and phrasing. We did not use AI assistants for coding, brainstorming, research design, data analysis, interpretation of results, or any other critical intellectual contribution.

\end{document}